\documentclass{article} 
\usepackage{arxiv}
  
\usepackage{amsmath,amsfonts,amssymb,bm}
\usepackage{mathtools}

\def\1{\bm{1}}

\newcommand{\innp}[2]{\left\langle#1, #2 \right\rangle}
\newcommand{\norm}[1]{\left\| #1 \right\|}

\def\va{{\bm{a}}}
\def\vb{{\bm{b}}}

\def\vd{{\bm{d}}}

\def\vm{{\bm{m}}}

\def\vu{{\bm{u}}}
\def\vvv{{\bm{v}}}

\def\vy{{\bm{y}}}
\def\vz{{\bm{z}}}

\newcommand{\E}{\mathbb{E}}

\newcommand{\R}{\mathbb{R}}
\newcommand{\KL}{\ensuremath{\mathrm{KL}}}
\newcommand{\Var}{\mathrm{Var}}

\DeclareMathOperator*{\argmin}{arg\,min}

\newcommand{\SinkDivEps}{\ensuremath{D_{\varepsilon}}\xspace}
\newcommand{\SinkObjEps}{\ensuremath{S_{\varepsilon}}\xspace}

\newcommand{\budget}{\ensuremath{B}\xspace}
\newcommand{\npoints}{\ensuremath{n}\xspace}
\newcommand{\Hypersimplex}{\ensuremath{\Delta_{\npoints,\budget}}\xspace}
\newcommand{\HypersimplexVertices}{\ensuremath{\mathcal{V}_{\npoints,\budget}}\xspace}
\newcommand{\Simplex}[1]{\ensuremath{\Delta_{#1}}\xspace}
\newcommand{\selection}{\ensuremath{\mathbf{x}}\xspace}
\newcommand{\TransportPlan}{\ensuremath{\bm{\pi}}\xspace}

\newcommand{\Dataset}{\ensuremath{\mathcal{D}}\xspace}
\newcommand{\EmpiricalDist}{\ensuremath{\hat{\mu}}\xspace}
\newcommand{\EmpiricalDistDataset}{\ensuremath{\EmpiricalDist^{\Dataset}}\xspace}

\newcommand{\LipschitzSwap}{\ensuremath{L^{\mathrm{swap}}_{\varepsilon,\lambda}}\xspace}

\newcommand{\PotentialSinkf}{\ensuremath{\mathbf{f}}\xspace}
\newcommand{\PotentialSinkg}{\ensuremath{\mathbf{g}}\xspace}
\newcommand{\PotentialSinkp}{\ensuremath{\mathbf{p}}\xspace}

\newcommand{\FWSwap}{\textbf{FW-Swap}\xspace}

\newcommand{\learningparam}{\ensuremath{\bm{\theta}}\xspace}

\newcommand{\swapset}{\ensuremath{\mathrm{swap}}\xspace}

\usepackage{hyperref}
\usepackage{url}
\usepackage{xspace}
\usepackage{natbib}
\usepackage{enumitem}
\usepackage{booktabs}
\usepackage{algorithm}
\usepackage{algpseudocode}
\usepackage{wrapfig}
\usepackage{caption}
\usepackage[noabbrev,capitalize,nameinlink]{cleveref}[0.21]

\usepackage{amsthm}
\usepackage{thmtools}
\usepackage[noabbrev,capitalize,nameinlink]{cleveref}[0.21]

\usepackage{tikz}
\usetikzlibrary{arrows.meta}

\usepackage{thm-restate}

\declaretheorem[name=Theorem]{theorem}
\declaretheorem[name=Proposition]{proposition}
\declaretheorem[name=Assumption]{assumption}

\declaretheorem[name=Remark]{remark}
\declaretheorem[name=Lemma]{lemma}

\crefname{assumption}{Assumption}{Assumptions}

\title{Low-Budget Active Learning through Entropic Optimal Transport}

\author{\hspace{1mm}Rim Hajal\\
	Univ.~Grenoble Alpes, CNRS, Inria, LIG, LJK\\
	\texttt{rim.alhajal@univ-grenoble-alpes.fr} \\
	\And
	\hspace{1mm}Mathieu Besançon\\
	Univ.~Grenoble Alpes, Inria, CNRS, LIG\\
	\texttt{mathieu.besancon@inria.fr} \\
    \And
	\hspace{1mm}Jérôme Malick\\
	Univ.~Grenoble Alpes, CNRS, LJK\\
	\texttt{jerome.malick@univ-grenoble-alpes.fr} \\
}
\date{}

\begin{document}

\maketitle

\begin{abstract}
We consider low-budget active learning, which consists of selecting a limited number of points, the coreset, such that a model can be trained to high accuracy on the selection only.
This problem is particularly relevant in contexts where labeling requires costly expert intervention, as in medical applications. 
We leverage features extracted from a pretrained self-supervised model to represent the data, and perform coreset selection directly in this feature space. 
In this paper, we use entropic optimal transport, specifically the Sinkhorn divergence, as the coreset selection criterion, which first allows us to get dimension-free sample complexity results, and second admits computationally efficient gradient evaluations. 
This opens the way to using gradient-based algorithms to rapidly compute solution candidates, further improved by a swap-based local search, with guarantees on the solution quality. Experiments on image benchmarks and medical datasets show that our method outperforms state-of-the-art heuristics in low-budget settings. 
\end{abstract}

\section{Introduction}
Active learning is a machine learning paradigm designed to reduce the amount of labeled data required to train a model by selecting the most informative samples for annotation. Instead of labeling a large dataset uniformly,
the algorithm selects a subset of unlabeled data points whose labels are expected to provide the greatest improvement to the model \citep{settles2009active}.
In some key applications, labeling is expensive, and by strategically querying labels for a small number of examples, active learning can achieve performance comparable to that of training on a much larger labeled dataset \citep{wang2024comprehensive}. 
In the low-budget setting, the problem relies critically on defining a small set of informative samples. 
Self-supervised learning models have provided increasingly powerful representations in recent years, in which train to map data points to a feature space where Euclidean distances can be used to evaluate their proximity \citep{gui2024survey}.
This improvement helped low-budget active learning reach higher levels of accuracies.

Active learning strategies broadly fall into two categories,
\emph{uncertainty-based} methods \citep{settles2009active} which query points where the current model is least confident but require an already-trained model and typically a high labeling budget, and \emph{representation-based} methods which select a \emph{coreset}, a subset of the data that best summarizes the whole dataset, independently of the current learning task. Classical approaches include k-medoids \citep{kaufman2009finding} and k-centers \citep{sener2018active}, and more recently, TypiClust \citep{hacohen2022active}, ProbCover \citep{yehuda2022active}, and MaxHerding \citep{bae2024generalized}.
The challenge of representation-based active learning thus lies in defining a criterion that captures the most representative data points for the learning task.

Optimal Transport (OT) provides a natural framework for defining a criterion that measures the discrepancy between the coreset and the whole dataset for the downstream learning task.
Given two empirical distributions supported on the dataset,
OT defines the minimum cost required to transport mass from one distribution to the other.
This cost, known as the Wasserstein distance, can be expressed as a linear program over transport plans (i.e., probability distributions in the product space subject to marginal constraints),
and was exploited in particular in \citet{mahmood2022lowbudget} in the context of active learning.

One main reason for the success of OT in learning is the use of the entropic regularization of OT, pioneered by \citet{cuturi2013sinkhorn} and presented across a range of applications in \citet{peyre2019computational}.
Entropic OT first offers an efficient and scalable method, through the so-called Sinkhorn algorithm, to approximate Wasserstein distances. 
Furthermore, entropic OT has been investigated in its own right, notably offering better statistical properties than the Wasserstein distance \citep{genevay2019sample,rigollet2025sample}. 
Related work on the use of entropic OT for coreset selection include \citet{kokot2025coreset,yin2025wasserstein}. First, \citet{kokot2025coreset} consider a similar coreset selection problem for a compression task, which differs from the low-budget active learning setting. Second, \citet{yin2025wasserstein} propose to use the Sinkhorn loss for coreset selection across both low- and high-budget regimes, without addressing the combinatorial hardness of the discrete selection problem in the low-budget setting. Our work is then complementary to theirs by investigating the symmetric Sinkhorn divergence and developing algorithmic tools tailored for low budgets.

\paragraph*{Contributions.}
We provide an entropic OT framework for coreset selection in low-budget active learning, with the following central contributions:
\begin{itemize}[leftmargin=0.5cm]
\item \textbf{Sinkhorn divergence criterion.}
We propose to measure the quality of the selection by the 
Sinkhorn divergence. 
We illustrate the fact that the bias of the usual Sinkhorn \emph{cost} of entropic OT 
may lead to degenerate selections, and 
that the Sinkhorn \emph{divergence} (\citet{genevay2018learning,feydy2019interpolating}), by correcting this bias, recovers the natural selections. 
We then establish a dimension-free generalization bound for general models trained on the resulting coreset.
\item \textbf{Algorithm and guarantees.} 
We propose an algorithm ($\FWSwap$) to approximately solve the selection problem, by leveraging the computationally cheap gradients of the Sinkhorn divergence. More specifically, $\FWSwap$
warm-starts a first-order swap-based local search, with candidates from
the active set produced by a Frank-Wolfe algorithm. 
We establish a guarantee on a worst-case suboptimality bound for this local algorithm.
\item \textbf{Experiments.} We evaluate the empirical performance of our approach on vision and EEG benchmarks,
against previously proposed methods and show competitive or superior accuracy across all low-budget settings. 
\end{itemize}

\section{Active learning from the Sinkhorn divergence}

We detail in this section our approach using the Sinkhorn divergence for coreset selection. We start by introducing, in \cref{sec:prelim}, the notions of OT used in our developments and, in \cref{sec:framework}, our active set framework. We explain our methodological approach by discussing, in \cref{sec:sinkbiasal}, a revealing example and by presenting, in \cref{sec:genbound}, our generalization results.

\subsection{Preliminaries on optimal transport}\label{sec:prelim}

We quickly recall some essential notions of OT;we refer to \citet{peyre2019computational} for a comprehensive introduction to computational optimal transport and to \citet{chewi2025statistical} for its use in statistics.

The optimal transport distance (or Wasserstein distance) between two measures $\mu$ and $\nu$ quantifies the minimum transport cost to displace the mass from $\mu$ to $\nu$ given a ground cost $c$. Mathematically,
\begin{align*}
W(\mu, \nu) = \min_{\TransportPlan \in \Pi(\mu,\nu)} \int c(\vy,\vz)\, d\TransportPlan
\end{align*}
where $\Pi(\mu,\nu)$ is the set of ``transport plans'' (or joint distributions on the product space) 
such that the first marginal is $\mu$ and the second marginal is $\nu$. 
The entropic OT cost is then introduced as the regularization of the Wasserstein distance by the Kullback-Leibler divergence from the product measure $\pi_0=\mu\otimes \nu$: for a regularization parameter $\varepsilon>0$, 
\begin{align*}
\SinkObjEps(\mu, \nu) = \min_{\TransportPlan \in \Pi(\mu,\nu)} \int c(\vy,\vz)\, d\TransportPlan + \varepsilon \KL(\TransportPlan,\mu\otimes \nu).
\end{align*}
Though computationally more tractable thanks to the so-called Sinkhorn algorithm, the entropic OT cost $\SinkObjEps(\mu, \nu)$ is not a distance nor even a divergence, unlike $W(\mu,\nu)$ (assuming that $c$ defines a metric).
In particular, it introduces a significant bias, since one notably has $\SinkObjEps(\mu, \mu) > 0$. To correct this bias, the Sinkhorn divergence was introduced in
\citet{genevay2018learning} (and further studied in
\citet{feydy2019interpolating,sejourne2019sinkhorn,lavenant2024riemannian}):
\begin{align}\label{eq:div}
\SinkDivEps(\mu, \nu) = \SinkObjEps(\mu, \nu) - \frac{1}{2} \SinkObjEps(\mu, \mu) - \frac{1}{2}\SinkObjEps(\nu, \nu).
\end{align}
We finish these recalls by the case of empirical measures over $\npoints$ points, identified with elements of $\Simplex{n}$ the probability simplex of dimension $n-1$. In the sequel, 
we use the shorthand $[n] = \{1\dots n\}$. For $\va \in \Simplex{n}$, $\vb \in \Simplex{m}$, and the cost matrix of entries $C_{ij}=c(\va_i,\vb_j)$, computing $W(\va,\vb)$ amounts to solving a linear optimization problem, which can be expressed in primal form over the transport plan \TransportPlan, and in dual form over the dual variables $\PotentialSinkf \in \R^n$ and $\PotentialSinkg \in \R^m$ (called potentials):
\begin{align*}
W(\va,\vb) = \max_{\PotentialSinkf \in \R^n, \PotentialSinkg \in \R^m} \; & \innp{\PotentialSinkf}{\va} + \innp{\PotentialSinkg}{\vb} \\
\text{s.t. }\;  & \PotentialSinkf_i + \PotentialSinkg_j \leq C_{ij} \quad \forall i \in [n], j \in [m]
\end{align*}
Similarly, for the entropic OT cost, we have
\begin{align*}
\SinkObjEps(\va,\vb) = \max_{\PotentialSinkf \in \R^n, \PotentialSinkg \in \R^m} \innp{\PotentialSinkf}{\va} + \innp{\PotentialSinkg}{\vb} - \varepsilon\!\!\sum_{(i,j)\in[\npoints] \times [m]}\!\! \va_i \vb_j  \left(  \exp\left( \frac{1}{\varepsilon} (\PotentialSinkf_i + \PotentialSinkg_j - C_{ij})  \right) - 1 \right).
\end{align*}
We will use that $\SinkObjEps$ is a differentiable function and that its gradient
is given by
$\nabla \SinkObjEps(\va,\vb) = \left(\PotentialSinkf^*, \PotentialSinkg^*\right)$, for the optimal solutions $\PotentialSinkf^*$, $\PotentialSinkg^*$ of the above problem (under the normalization $\sum_{i} \PotentialSinkf^*_i = 0$). Finally, considering also $\PotentialSinkp^*$ as the optimal potential for the symmetric problem $\SinkObjEps(\va,\va)$, 
the gradient of the Sinkhorn divergence with respect to its first marginal is given by
\begin{align}\label{eq:sinkdivgradient}
\nabla_{\va} \SinkDivEps(\va,\vb) = \PotentialSinkf^* - \PotentialSinkp^*.
\end{align}
This gradient expression (obtained as a by-product of the two runs of the Sinkhorn algorithm \citep{feydy2019interpolating}) will be a key ingredient in the algorithm we propose in\;\cref{sec:FWSwap}.

\subsection{Entropic optimal transport framework for active learning}\label{sec:framework}

Consider a dataset of $\npoints$ points $\Dataset = \{\vy_i\}_{i\in [\npoints]}$ and 
the empirical distribution over this dataset:
\[
~~~~\EmpiricalDistDataset = \sum_{i\in[n]} \frac{1}{n}\delta_{\vy_i}
\qquad\text{(where $\delta_{\vy}$ is the Dirac measure at the point $\vy$)}
\]
For $\selection \in \{0,1\}^\npoints$,  we introduce $\EmpiricalDist_{\selection}$ as the empirical distribution over points whose incidence vector is\;$\selection$.
Given a budget $\budget \leq \npoints$, we write $\HypersimplexVertices$
the set of $\{0,1\}$-vectors of dimension $\npoints$ with $\budget$ nonzero entries, and $\Hypersimplex$ for the hypersimplex in dimension $\npoints$ of parameter $\budget$:
\begin{align*}
\HypersimplexVertices = \Big\{ \selection \in \{0,1\}^{\npoints} \, : \, \sum_{i \in [\npoints]} \selection_i = \budget \Big\} \quad\text{and}\quad 
\Hypersimplex = \left\{ \selection \in [0,1]^{\npoints} \, : \, \sum_{i \in [\npoints]} \selection_i = \budget \right\}.
\end{align*}

The method we develop in this paper uses the Sinkhorn divergence $\SinkDivEps$ for coreset selection. More specifically, we propose to select $B$ points within the dataset $\Dataset$, as the best selection vector $\selection\in \HypersimplexVertices$ such that $\EmpiricalDist_{\selection}$ minimizes the Sinkhorn divergence from $\EmpiricalDist_{\selection}$ to the empirical measure $\EmpiricalDistDataset$. 
The coreset selection problem thus writes:
\begin{align}\label{eq:genericlearning}
\underset{\selection \in \HypersimplexVertices}{\min}\; \SinkDivEps(\EmpiricalDist_{\selection}, \EmpiricalDistDataset).
\end{align}

Note that this formulation relates to existing approaches from the literature for the coreset selection problem. First, replacing $\SinkDivEps$ by $W$ with an Euclidean ground cost $c$, we recover the approach of \citet{mahmood2022lowbudget}. Moreover, this OT-based formulation also captures two seemingly unrelated previous methods: the problems underlying ProbCover \citep{yehuda2022active} and k-medoid \citep{kaufman2009finding} can indeed be recast as a similar problem with a suitable ground cost. This is formalized in \cref{prop:probcoverOT} in \cref{apx:additionalproofs}.

The active learning problem can thus be formalized as the combinatorial problem~\eqref{eq:genericlearning}, which aims at minimizing, over the discrete set $\HypersimplexVertices$, the function $\SinkDivEps$, which is differentiable and convex in its first argument. In the next subsections, we explain why $\SinkDivEps$ is preferable over $\SinkObjEps$ and what generalization properties we can expect. 
\cref{sec:sol_methods} will then present our solution method. 
Before moving on, we highlight in the next remark that the combinatorial structure of the problem is essential, since 
its convex relaxation is trivial.

\begin{remark}\label{rem:relax}
In general, one way to approximate a combinatorial problem is to solve its convex relaxation: in our situation, this would mean replacing $\HypersimplexVertices$ with its convex hull, the hypersimplex~$\Hypersimplex$, 
\begin{align}\label{eq:relax}
    \min_{\selection \in \Hypersimplex} \SinkDivEps(\EmpiricalDist_{\selection}, \EmpiricalDistDataset).
\end{align}
However, we easily see that this relaxation is uninformative: solving this relaxation yields both a zero optimal value and a constant optimal solution, regardless of the instance. This is formalized,  for any divergence, in \cref{prop:convexrelax} in \cref{apx:additionalproofs}. Still, in \cref{sec:FWSwap}, we will use continuous optimization methods, but in an original way, to get relevant information on the combinatorial problem.
\end{remark}

\subsection{Sinkhorn bias on active learning instances}\label{sec:sinkbiasal}

In this section, we highlight why it is essential to use the Sinkhorn divergence $\SinkDivEps$ for our active learning problem, rather than simply the entropic OT cost $\SinkObjEps$.
Indeed, the Sinkhorn algorithm is known to introduce a bias \citep{feydy2019interpolating,sejourne2019sinkhorn}, corrected by the symmetrization of $D_{\varepsilon}$. 
Here we characterize the effect of the Sinkhorn bias on a simple parameterized family of instances (with two clusters and a single central outlier), showing that using $\SinkObjEps$ can be problematic for the coreset selection task and active learning in particular.

\begin{wrapfigure}{r}{0.48\textwidth}
\vspace{-15pt}
\centering
\begin{tikzpicture}[
      every node/.style={font=\small},
      arr/.style={-{Stealth[length=6pt]}, line width=1.2pt},
      dimarr/.style={<->, >=Stealth, thin},
]

\def\DeltaHeight{2.8}
\def\heightm{2.2}   
\def\heightone{0.9} 

\draw[thin, ->, >=Stealth] (-3.4,0) -- (3.4,0);

\draw[arr, blue!70!black] (-\DeltaHeight,0) -- (-\DeltaHeight,\heightm);
\node[below] at (-\DeltaHeight,0) {$-\Delta$};

\draw[arr, blue!70!black] (0,0) -- (0,\heightone);
\node[below] at (0,0) {$0$};

\draw[arr, blue!70!black] (\DeltaHeight,0) -- (\DeltaHeight,\heightm);
\node[below] at (\DeltaHeight,0) {$\Delta$};

\node[right=3pt] at (\DeltaHeight, \heightm/2) {$m$};
\node[right=3pt] at (0, \heightone/2) {$1$};
\end{tikzpicture}
\vspace{-15pt}
\caption{$\EmpiricalDistDataset$ for the $\texttt{CCO}(\Delta,m)$ instances. For $\budget\leq 2$, we would like to select one point in $-\Delta$ and another in $\Delta$ (the \emph{extreme} configuration).\label{fig:biasillustration}}
\vspace{-15pt}
\end{wrapfigure}

Let us consider the simple family of instances called $\texttt{CCO}(\Delta,m)$ defined by two clusters of $m \geq 2$ points each, symmetric around the origin and at distance $\Delta$ from the origin, and a single point, the outlier, at the origin,
see \cref{fig:biasillustration} for the illustration.
The empirical measure is 
\[
\EmpiricalDistDataset = \frac{1}{2m+1} \delta_{0} + \frac{m}{2m+1} \delta_{-\Delta} + \frac{m}{2m+1} \delta_{\Delta}.
\]

There are four distinct configurations up to symmetries:
the \emph{central} configuration with measure $\alpha_C$ 
corresponds to the mass concentrated at the origin,
the \emph{left} configuration with measure $\alpha_L$ 
corresponds to the mass evenly split between the origin and the left cluster,
the \emph{extreme} configuration with measure $\alpha_E$ 
corresponds to the mass evenly split between the two clusters,
the \emph{one-sided} configuration with all of the mass concentrated on the left cluster.
The last one is always trivially dominated and we discard it.
With $m \gg 1$, this instance corresponds to a dataset with two clusters and one central outlier separated from both clusters.
\cref{prop:entropicbiassummary} shows that using the Sinkhorn cost with a large regularization $\varepsilon$ misses the right configuration.

\begin{proposition}[Entropic bias on \texttt{CCO} instances]\label{prop:entropicbiassummary}
Take $B\leq 2$ and $\norm{\vy_1 - \vy_2}^2_2$ as the ground cost.
For any value $\Delta > 0, m \geq 2$, the $\texttt{CCO}(\Delta, m)$ instance has the following properties:
\begin{itemize}[leftmargin=0.5cm]
\item For any $\varepsilon > 0$, selecting the coreset with $\SinkDivEps$ favors the extreme configuration $\alpha_E$, which is the desirable outcome, and identical to selecting the coreset with Wasserstein.
\item However there exists a finite $\bar\varepsilon > 0$ such that for any $\varepsilon > \bar\varepsilon$,
selecting the coreset based on $S_\varepsilon$ leads to favoring the central configuration $\alpha_C$. 
\end{itemize}
\end{proposition}

In \cref{apx:entropicbiasstudy}, we present the quantitative version of this proposition and its proof, precisely comparing the three configurations above. Although intuitive, this example illustrates and quantifies the effect of the Sinkhorn bias in our active learning setting.  

\subsection{Generalization bound for the Sinkhorn divergence coreset}\label{sec:genbound}

In this section, we establish a generalization result for our approach, based on the elegant dimension-independent complexity result of~\cite{rigollet2025sample} (in the case of the squared Euclidean norm as the ground cost). We provide here the framework and the main result, and we postpone to \cref{apx:gen_bound} the full proof.

We assume that the data is i.i.d~samples from a ``true'' data distribution $\mu$, i.e. 
$\Dataset = \{\vy_i\}_{i\in [\npoints]} \sim \mu$. 
For a parametric family of loss functions $f_{\learningparam}$ of parameter $\learningparam$, we define the risk of a parameter $\learningparam$ under a distribution $\nu$ as
\begin{align*}
R_\nu(\learningparam) = \E_{\vy \sim \nu}[f_{\learningparam}(\vy)]
\end{align*}
We then consider three risks: the true risk $R(\learningparam)$, the empirical risk $\hat{R}_n(\learningparam)$, and the coreset empirical risk $\hat{R}_{S^*}(\learningparam)$ (when $\nu=\hat{\mu}_{S^*}$ with the subset $S^* \subset [\npoints]$ supported on \budget).
As typical in active learning (see e.g., \citet{sener2018active}), 
we decompose the risk error of our method into a generalization term and a coreset loss term:
\begin{align}\label{eq:decomp}
R(\learningparam) - \E[\hat{R}_{S^*}(\learningparam)] = \E\left[ R(\learningparam) - \hat{R}_{\npoints}(\learningparam) \right] +
\E \left[ \hat{R}_{\npoints}(\learningparam) - \hat{R}_{S^*}(\learningparam) \right].
\end{align}
The next result shows that we can bound each term, independently of the dimension, when the loss functions are Lipschitz and the ground cost is the squared Euclidean norm.

\begin{restatable}[Generalization bound for the Sinkhorn divergence coreset]{theorem}{genbound}\label{thm:genbound}
Consider the Sinkhorn divergence $D_{\varepsilon}$ with ground cost $c(\vy,\vz)=\norm{\vy-\vz}_2^2$.
Suppose that the loss function $f_{\learningparam}$ is $L$-Lipschitz for any $\learningparam$.  
Then the coreset loss is bounded uniformly in $\learningparam$ by
\begin{align*}
\hat{R}_n(\learningparam) - \hat{R}_{S^*}(\learningparam) & \leq L \sqrt{ \SinkDivEps(\hat{\mu}_{S^*},\EmpiricalDistDataset) + \Var[\hat{\mu}_{S^*}] + \Var[\EmpiricalDistDataset] }.
\end{align*}
Moreover, if the underlying data distribution $\mu$ satisfies
$\norm{\vy}_2 \leq \frac12$ for $\mu$-almost every sample $\vy$,
then, for any estimator $\hat\learningparam$, the generalization with respect to the true risk is bounded by
\begin{align*}
\E\left[ R(\hat\learningparam) -  \hat{R}_{\npoints}(\hat\learningparam) \right] 
& \leq L \sqrt{ \frac{\kappa(\varepsilon)}{\npoints} + \frac{2 \npoints -1}{\npoints} \Var[\mu] } 
\end{align*}
where $\kappa(\varepsilon)>0$ is a constant depending only on the regularization parameter.
\end{restatable}

The first part of this theorem establishes that the coreset loss is upper-bounded by a term involving the Sinkhorn divergence, that is minimized during learning. The second part, based on recent sample complexity results from \citet{rigollet2025sample}, gives a dimension-free bound.
Putting the two bounds together thus gives a bound on the quantity in~\eqref{eq:decomp} and thus provides theoretical grounds to use the Sinkhorn divergence as a proxy for the representativeness of the coreset. 

The key hypothesis of this result is the Lipschitzness of the loss functions, which could be a restriction for a full learning model. However, 
low-budget active learning relies on two distinct learning phases: (1) the self-supervised learning model and (2) the final classifier trained on the small number of examples. 
Our coreset selection problem is solved in the feature space built by the self-supervised learning model and the Lipschitz bound is only needed on the final classifier.

\section{Solution methods for the Sinkhorn divergence coreset}\label{sec:sol_methods}

This section presents (in \cref{sec:FWSwap}) and analyzes (in \cref{sec:algoanalysis}) the proposed algorithm for solving the Sinkhorn divergence selection problem~\eqref{eq:genericlearning}. An analysis of the runtime of our algorithm is provided in \cref{sec:runtimeanalysis}. This algorithm is based on an original combination of gradient-based continuous optimization methods with an efficient swap heuristic.

\subsection{The Frank-Wolfe Swap Algorithm}\label{sec:FWSwap}

\begin{wrapfigure}{r}{0.5\columnwidth}
\begin{minipage}{0.48\columnwidth}
\vspace{-25pt}
\begin{algorithm}[H]
\caption{\FWSwap}\label{alg:fw-swap}
\small
\begin{algorithmic}[1]
    \Require{Cost matrix $C$, budget $\budget$, regularization $\varepsilon$}
    \State{Define $f(\selection) = \SinkDivEps(\EmpiricalDist_{\selection}, \EmpiricalDistDataset)$}
    \vspace{0.2em}
    \State{\textbf{(1) Continuous phase}}
    \State{Run FW for $\min_{\selection \in \Hypersimplex} f(\selection)$ for $T$ iterations}
    \State{Obtain fractional solution $\tilde{\selection}$ and active set $\mathcal{A}$}
    \State{\textbf{(2) Swap phase}}
    \State{Initialize candidates $\mathcal{C} \gets \{\}$}
    \For{$\va \in \mathcal{A}$}
        \State{Initialize $\selection \gets \va$}
        \State{Initialize a solution pool $\mathcal{S} \gets \{\va\} $}
        \While{no cycle created}
        \State{$\selection_{\mathrm{new}} \gets \argmin_{\vvv \in \swapset(\selection)} \innp{\nabla f(\selection)}{\vvv}$}
        \If{$\selection_{\mathrm{new}} \in \mathcal{S}$}
            \State{\textbf{break}}
        \EndIf{}
        \State{$\mathcal{S} \gets \mathcal{S} \cup \{\selection_{\mathrm{new}}\}$}
        \State{$\selection \gets \selection_{\mathrm{new}}$}
        \EndWhile{}
        \State{$\mathcal{C} \gets \mathcal{C} \cup \{ \argmin_{\selection \in \mathcal{S}} f(\selection) \}$}
    \EndFor{}
    \State{\Return{$\argmin_{\selection \in \mathcal{C}} f(\selection)$}}
\end{algorithmic}
\end{algorithm}
\end{minipage}
\vspace{-25pt}
\end{wrapfigure}
The Frank-Wolfe (FW) algorithm \citep{jaggi2013revisiting} produces sparse solutions to convex constrained optimization, by constructing each iterate as a convex combination of a limited number of extreme points in the constraint set. 
In our situation, despite \cref{rem:relax}, we will apply a variant of FW to the convexified problem~\eqref{eq:relax}, to obtain 
a relevant \emph{active set} $\mathcal{A}$, a set of vertices whose convex combination construct the current iterate. 

We performed tests with random initialization candidates that yield worse results and get stuck in local minima, which we provide in the empirical analysis in \cref{apx:vertexswap}. Empirically, we find that the active set returned after running FW for a limited number of iterations always contains vertices with low divergence.
We then refine these good solutions, by running a local swap algorithm exploiting gradient information to move through vertices. 

Our algorithm $\FWSwap$ runs in two stages: first, in the \emph{continuous phase}, it executes the FW variant algorithm, blended pairwise conditional gradient \citep{tsuji2022pairwise},
for a fixed number of iterations on the hypersimplex, returning the active set.
Then, in the \emph{swap phase}, it runs a local swap algorithm: starting from a vertex,
the swap algorithm computes the best single swap minimizing in the direction of the gradient, 
jumping to an adjacent vertex of the vertex-edge graph of the hypersimplex, and stopping when it has created a cycle.
The vertex returned by each swap step is the one having the lowest divergence on the whole cycle.

We detail $\FWSwap$ in \cref{alg:fw-swap}. The description involves the operator $\swapset$, which, given $\selection \in \HypersimplexVertices$, produces $\swapset(\selection) \subset \HypersimplexVertices$ the subset of vertices that differ from \selection by exactly two entries,
i.e.~by swapping one nonzero entry with a zero entry. 
The proposed algorithm combines core ideas from two groups of methods:
a) local search algorithms for combinatorial problems typically evaluate the objective function on all neighbors
of the current vertex, including for instance the analysis of \citet{arya2001local} for k-median and facility location;
b) ``follow-the-gradient'' methods which use the linear minimization oracle of the complete polytope as studied in \citet{baes2012minimizing}, akin to a Frank-Wolfe algorithm with a unit step size. 
Taking the useful aspects of these methods, our algorithm avoids their weaknesses for our context.\footnote{We also evaluated empirically: zero-th order local search which is prohibitively expensive,
and follow-the-gradient which never provided improvements over the candidates from the active set.}

\begin{remark}[Additional mixture term]
The Sinkhorn divergence is smooth only for $\selection$ in the relative interior of the hypersimplex; as a consequence, at the boundary,
the potentials cannot be obtained as fixed points of the Sinkhorn iterations.
In \citet{yin2025wasserstein}, this issue is handled by a smooth optimization scheme that short-circuits the Sinkhorn algorithm. 
In our work, we prefer to stabilize the evaluations with a mixture term $\lambda$, which we rely on both numerically and for the analysis of \cref{sec:algoanalysis}.
The empirical distribution on the selection is then defined as
$\alpha_{\selection}^{\lambda} = (1-\lambda) \EmpiricalDist_{\selection} + \lambda \EmpiricalDistDataset$.
This stabilization avoids issues in the computation of the potentials
at the boundary of the hypersimplex; 
For simplicity of exposition, we present $\FWSwap$ in \cref{alg:fw-swap} without\;$\lambda$.
\end{remark}

\subsection{Approximation guarantees}\label{sec:algoanalysis}

In this section, we analyze the swapping procedure of $\FWSwap$ that is run on the active set constructed by an initial run of the Frank-Wolfe algorithm. The next theorem establishes a suboptimality guarantee on the quality of the computed solutions.
We defer the proof and relevant empirical experiments to \cref{apx:vertexswap}.

\begin{restatable}[Suboptimality of Swap Descent on Sinkhorn Divergence]{theorem}{swapdescent}\label{thm:swapdescent}

Consider the execution of the swap algorithm with mixture parameter $\lambda \in (0,1)$ and denote by $\bar\selection$ its output. 
Consider $\selection^*$ a minimizer of the selection problem~\eqref{eq:genericlearning},  
and $k^*\leq \budget$ the number of swaps required from $\bar\selection$ to $\selection^*$.
There exists a smoothness constant $\LipschitzSwap$ depending on $\varepsilon$ and $\lambda$,
such that the following suboptimality bound holds:
\begin{align}
\SinkDivEps(\EmpiricalDist_{\bar\selection}, \EmpiricalDistDataset) - \SinkDivEps(\EmpiricalDist_{\selection^*}, \EmpiricalDistDataset)
 & \leq \frac{k^* \LipschitzSwap}{2} + 4\lambda \left(1 - \frac{\budget}{\npoints} \right) \norm{C}_{\infty}.\label{eq:swapsuboptimality}
\end{align}
\end{restatable}

The bound on the quality of the solution thus decomposes into the first term bounding suboptimality of the local search and the second term controlling the bias introduced by the mixture term $\lambda$. 
Considering $\lambda \ll 1$, our bound directly relates the suboptimality of the output $\bar\selection$ to $k^*\leq \budget$, the number of swaps required from $\bar\selection$ to the optimal solution $\selection^*$.
The expression of the constant $\LipschitzSwap$ appearing in this bound is provided with the full proof in \cref{apx:vertexswap}.

\subsection{Runtime analysis}\label{sec:runtimeanalysis}
\paragraph{Per iteration cost.} Each swap iteration requires one Sinkhorn OT computation, while the symmetric OT correction is comparatively inexpensive \citep{feydy2019interpolating}. Thus, the Sinkhorn OT computation dominates the per iteration cost and has complexity $O(\npoints^2)$. In contrast, a local search strategy requires $O\left(\budget(\npoints - \budget)\right)$ function evaluations per iteration, which was prohibitively expensive in our experiments.

\paragraph{Number of swap iterations.} This is observed empirically and explained by a trajectory consisting of a long tail followed by a short cycle. In particular, the descent inequality bounds the tail in terms of the objective range and the swap gap, while the cycle can only persist once the swap gap has reached the scale of the curvature constant. We further make the following assumption.

\begin{restatable}[Existence of a core]{assumption}{core}\label{ass:core}
Note $S$ the support of $\selection$. There exists a set $S_0$ with $\lvert S_0\rvert=\budget_0<\budget$ such that, along the swap trajectory, every point in $S_0$ is acquired before any point outside $S_0$, and once acquired, no point in $S_0$ is subsequently removed.
\end{restatable}

Under \cref{ass:core}, the core is fully acquired in at most $\budget_0$ iterations, after which the optimization reduces to a hypersimplex with $\npoints-\budget_0$ candidate points and a budget $\budget-\budget_0$. The preceding tail and cycle bounds then apply to this reduced problem.
The detailed analysis is given in \cref{apx:runtime}.

\section{Computational experiments}\label{sec:emp-results}
In this section, we evaluate the empirical performances of our method \FWSwap, and we compare it to previously proposed methods. 

\subsection{Experimental setup}

\paragraph{Datasets.} We evaluate our active learning methods on the following datasets:
\begin{itemize}[leftmargin=0.2cm]
\item \textbf{STL-10} \citep{coates2011analysis}: a standard image classification benchmark containing $13\;000$ labeled images across 10 classes.
\item \textbf{SVHN} \citep{netzer2011reading}: a digit classification dataset consisting of $99\;289$ real-world images of house numbers. We use the standard training and test sets with 10 classes.
\item \textbf{GasHisSDB} \citep{hu2022gashissdb}: a gastric histopathology image dataset containing 2 classes: normal and abnormal tissue samples. 
We use $33\;284$ images corresponding to $160\times160$ cropped patches.
\item \textbf{EEGMMIDB} \citep{zhang2021deep}: an EEG dataset containing recordings from 109 subjects with $3\;131$ recordings in total. We consider the binary classification task corresponding to left- versus right-hand movement.
\item \textbf{WBCIC-SHU} \citep{yang2025multi}: another EEG dataset containing recordings from 52 subjects with $30\;591$ recordings in total, with the same binary classes as EEGMMIDB.
\end{itemize}

\paragraph{Phase 1: Feature learning.}
We first train a self-supervised learning model on the raw input data (images or signals) without their labels and extract the features\footnote{The choice of feature representation is critical, as the quality of the learned embedding directly impacts downstream performance. In particular, using raw images systematically leads to very low accuracies, consistent with previous studies \citep{yehuda2022active}. We assess the effect of the features extracted from SimCLR or DINOv3 in \cref{apx:simclr-vs-dino}.} then passed to the active learning methods. 
We rely on either SimCLR \citep{chen2020simple} or DINOv3 \citep{simeoni2025dinov3} for images.
Specifically, we use the ViT-S+/16 distilled model pretrained on web dataset LVD-1689M with DINOv3 for the STL-10 dataset. The frozen pretrained weights allow us to directly access the backbone, and thus, the feature space, without any training or fine tuning.
For lower resolution image datasets ($32 \times 32$), such as GasHisSDB and SVHN, we prefer to fine tune a SimCLR model, as initial observations show that it provides more suitable feature representations.
We use the SimCLR-based model implemented in the \texttt{SelfEEG} library \citep{DelPup2024} for the two EEG datasets.

\paragraph{Phase 2: Coreset selection.} 
We compare our algorithm $\FWSwap$ to existing approaches, specifically a random selection baseline, the classical heuristics k-medoids and k-centers,
and recent low-budget active learning methods: MaxHerding \citep{bae2024generalized}, TypiClust \citep{hacohen2022active}, and ProbCover \citep{yehuda2022active}.
We also experimented with the Wasserstein integer approach of \citet{mahmood2022lowbudget}, but consistently with their results and those of \citet{yehuda2022active}, we stop the method after three hours without obtaining competitive solutions, see appendix \cref{apx:simclr-vs-dino}.
ProbCover's distance hyperparameter $\delta$ is set to the mean of the entries of the cost matrix, as tuning this parameter requires grid search on the problem itself without good heuristics.
We fix the parameter associated to MaxHerding to one to match the original paper \citep{bae2024generalized}.


We take the ground cost matrix $C$ defined by the Euclidean distance between feature vectors. 
For the regularization parameter, we take a fixed value $\varepsilon=1.5$ for the vision tasks and $\varepsilon=0.5$ for the EEG tasks. Preliminary experiments with values in $[0.05, 1.5]$ show that, within this range, the parameter does not impact our methods. 
However, larger or smaller values for $\varepsilon$ drastically change the performance\footnote{It is well-known in OT that extreme values of $\varepsilon$ change the geometry of the problem. Putting our numerical observations in perspective with~\cite{feydy2019interpolating}, 
this suggests that true interpolation between extreme cases (Wasserstein and maximum-mean discrepancy) is the suitable regime for the coreset selection problem. This is also the regime used by \citet{rigollet2025sample} for their sample complexity result.}.

\paragraph{Phase 3: Active learning.} 
We train a linear classifier on the selected subsets of the training data and evaluate the classification accuracy on the associated test sets. The final active learning classifier is implemented using \texttt{Flux} v0.16.10 \citep{Flux.jl-2018} and consists of a single linear layer mapping the feature representations to the target classes. Model training is performed using a momentum-based optimizer defined with a learning rate $0.1$ and a momentum coefficient $0.9$.  The classifier is trained for 30 epochs at each active learning round using mini-batches of size 32. Training relies on the logit cross-entropy loss.

\paragraph{Hardware specification.}
The experiments are run on a cluster of Debian servers with NVIDIA GTX 1080 Ti GPUs and Intel Xeon E5-2620 v4 CPUs with 256Go of RAM, using \texttt{Julia} version 1.12.5 \citep{bezanson2012julia}.

\subsection{Empirical results}\label{sec:main-results}

We now present the main experimental results (additional results are provided in the appendix \cref{apx:additionalexperiments}). 
The figures below show the accuracies across the different budgets and methods for STL-10 (\cref{fig:stl-results}), for SVHN (\cref{fig:svhn-results}) and for GasHisSDB \cref{fig:gashis-results}, and the two tables present the results for both EEG datasets (\cref{tab:eeg-results,tab:wbcic-shu-results}). 
The reported values correspond to the mean and standard deviations across $10$ seeds. 
The seeds affect both randomness in the learning phase and selection from randomized methods (k-medoids, k-center, and random selection). 

\begin{figure}[hbt!]
\captionsetup{font=tiny}
    \centering
    \begin{minipage}{0.49\textwidth}
        \centering
        \includegraphics[width=\linewidth]{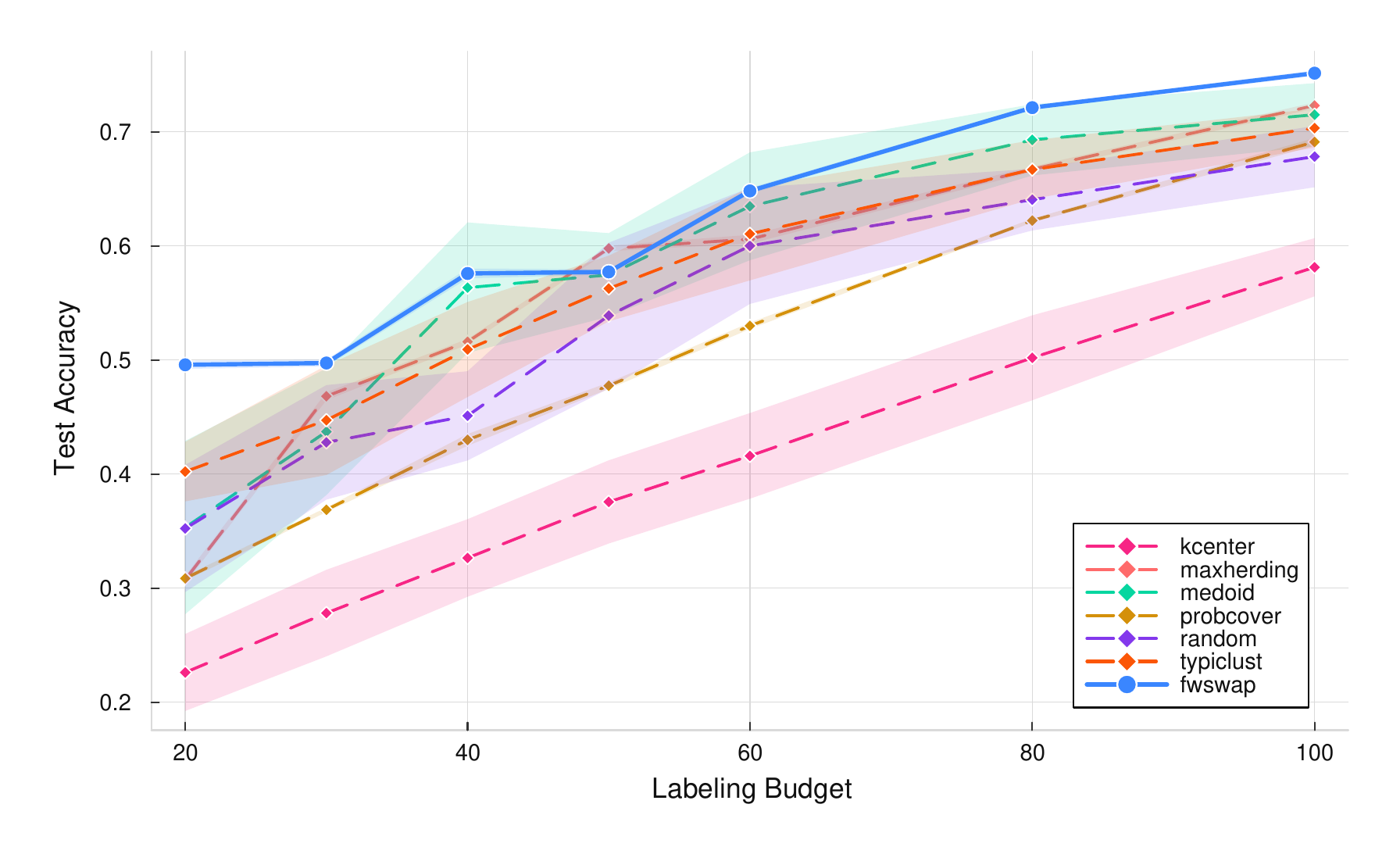}
        \caption{Test accuracy (mean $\pm$ std) on SVHN.}\label{fig:svhn-results}
    \end{minipage}
    \hfill
    \begin{minipage}{0.49\textwidth}
        \centering
        \includegraphics[width=\linewidth]{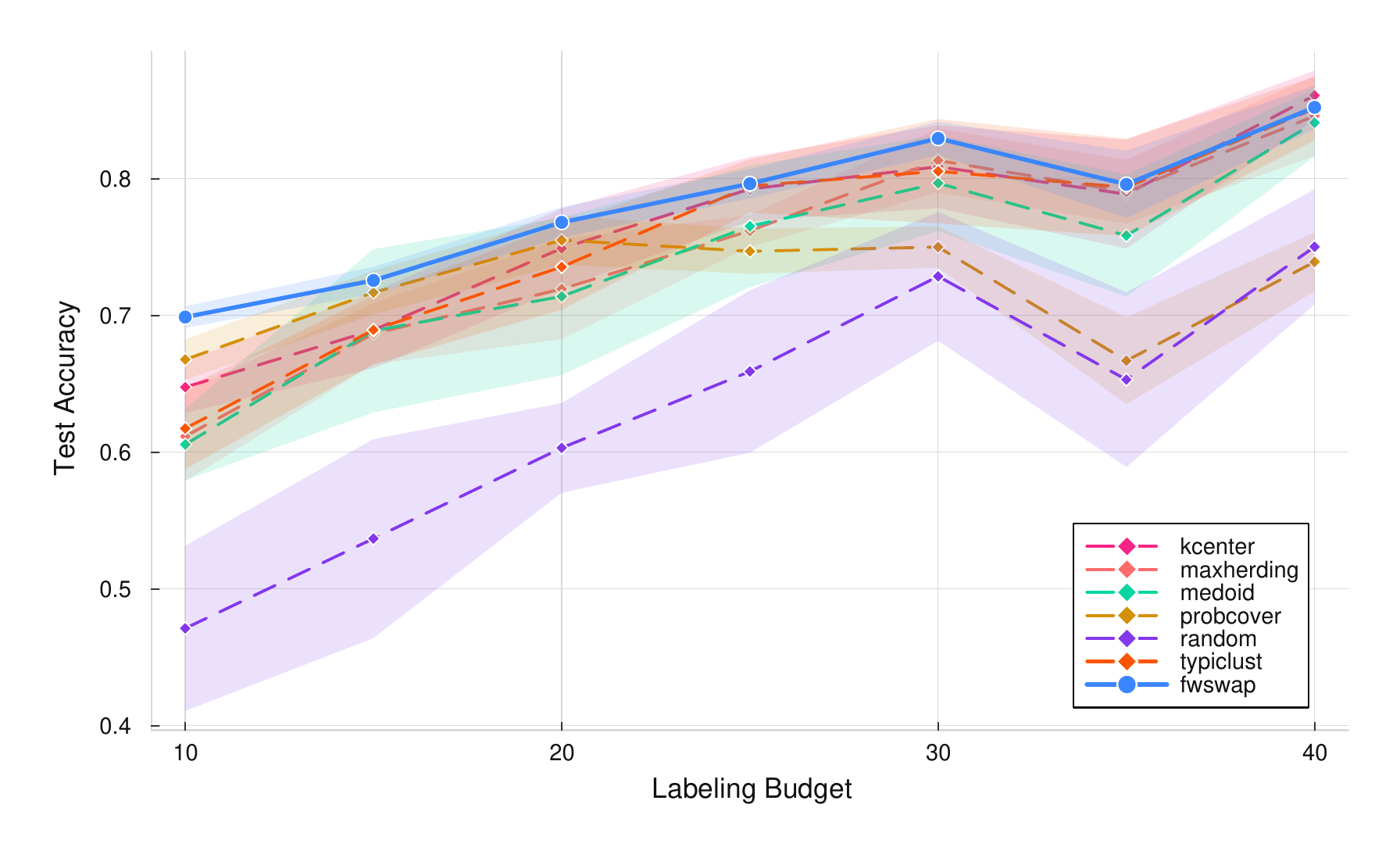}
        \caption{Test accuracy (mean $\pm$ std) on STL-10.}\label{fig:stl-results}
    \end{minipage}
\end{figure}

\begin{wrapfigure}{r}{0.5\linewidth}
\captionsetup{font=tiny}
    \includegraphics[width=\linewidth]{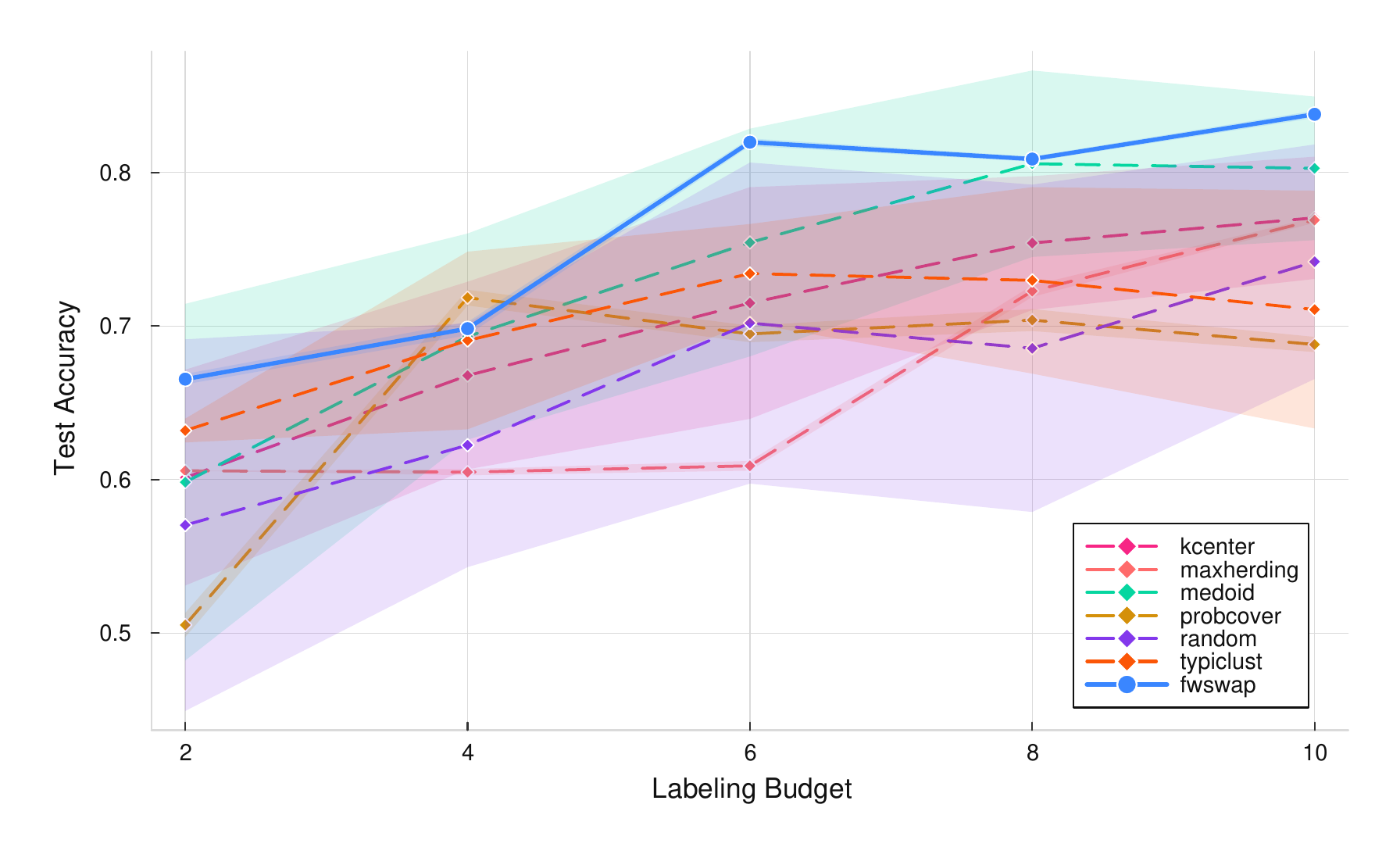}
    \caption{Test accuracy (mean $\pm$ std) on GasHisSDB.}\label{fig:gashis-results}
    \vspace{-15pt}
\end{wrapfigure}

For the SVHN dataset, shown in \cref{fig:svhn-results}, however, we consistently achieve the best performance across all low budget settings.
For the STL-10 dataset, as shown in \cref{fig:stl-results}, $\FWSwap$ yields the best results on low budgets. 
For higher budgets, starting $\budget=40$, $\FWSwap$ is comparable to the best methods.  
Results for more budgets are provided in \cref{apx:additionalexperiments}.

For the medical GasHisSDB dataset in \cref{fig:gashis-results}, our method achieves the best results across small budgets. Note that the original paper for the dataset reports an accuracy of $85.99\%$ when training the full dataset using color-histogram features and a Random Forest classifier \citep{hu2022gashissdb}. We achieve similar results with only $\budget=10$ images.

On both EEG datasets, $\FWSwap$ achieves consistently better accuracy on all budgets as seen in \cref{tab:eeg-results,tab:wbcic-shu-results} with a remarkable difference at $\budget=2$. 


\begin{table}[hbt!]
\tiny
  \caption{Test accuracy (mean $\pm$ std) on EEGMMIDB dataset.}\label{tab:eeg-results}
  \centering
    \begin{tabular}{lccccc}
    \toprule
    \textbf{Method} & \textbf{2} & \textbf{4} & \textbf{6} & \textbf{8} & \textbf{10} \\
    \midrule
    FW-Swap  & $\mathbf{0.805 \pm 0.002}$ & $\mathbf{0.714 \pm 0.003}$ & $\mathbf{0.783 \pm 0.003}$ & $\mathbf{0.748 \pm 0.004}$ & $\mathbf{0.783 \pm 0.002}$ \\
    \midrule
    MaxHerding  & $0.487 \pm 0.002$ & $0.687 \pm 0.003$ & $0.749 \pm 0.002$ & $0.705 \pm 0.002$ & $0.714 \pm 0.003$ \\
    TypiClust & $0.501 \pm 0.002$ & $0.642 \pm 0.014$ & $0.674 \pm 0.006$ & $0.697 \pm 0.027$ & $0.732 \pm 0.020$ \\
    Medoid      & $0.588 \pm 0.106$ & $\mathbf{0.716 \pm 0.063}$ & $0.735 \pm 0.050$ & $0.744 \pm 0.048$ & $0.765 \pm 0.038$ \\
    k-Center    & $0.500 \pm 0.001$ & $0.547 \pm 0.058$ & $0.574 \pm 0.055$ & $0.592 \pm 0.057$ & $0.629 \pm 0.057$ \\
    ProbCover   & $0.515 \pm 0.001$ & $0.564 \pm 0.002$ & $0.645 \pm 0.002$ & $0.653 \pm 0.002$ & $0.659 \pm 0.002$ \\
    Random      & $0.535 \pm 0.078$ & $0.602 \pm 0.107$ & $0.660 \pm 0.085$ & $0.719 \pm 0.063$ & $0.744 \pm 0.033$ \\
    \bottomrule
    \end{tabular}
\end{table}

\begin{table}[hbt!]
\tiny
  \caption{Test accuracy (mean $\pm$ std) on the WBCIC-SHU Motor Imagery dataset.}\label{tab:wbcic-shu-results}
  \centering
    \begin{tabular}{lccccc}
    \toprule
    \textbf{Method} & \textbf{2} & \textbf{4} & \textbf{6} & \textbf{8} & \textbf{10} \\
    \midrule
    FW-Swap & $\mathbf{0.758 \pm 0.079}$ & $\mathbf{0.773 \pm 0.041}$ & $\mathbf{0.811 \pm 0.029}$ & $\mathbf{0.819 \pm 0.010}$ & $\mathbf{0.829\pm 0.033}$ \\
    \midrule
    MaxHerding  & $0.499 \pm 0.003$ & $0.460 \pm 0.012$ & $0.488 \pm 0.023$ & $0.538 \pm 0.026$ & $0.596 \pm 0.037$ \\
    TypiClust   & $0.499 \pm 0.008$ & $0.507 \pm 0.003$ & $0.511 \pm 0.022$ & $0.514 \pm 0.012$ & $0.556 \pm 0.066$ \\
    Medoid      & $0.550 \pm 0.114$ & $0.664 \pm 0.119$ & $0.757 \pm 0.061$ & $0.740 \pm 0.086$ & $0.737 \pm 0.090$ \\
    k-Center & $0.498 \pm 0.037$ & $0.521 \pm 0.025$ & $0.524 \pm 0.034$ & $0.515 \pm 0.029$ & $0.546 \pm 0.050$ \\
    ProbCover   & $0.503 \pm 0.008$ & $0.673 \pm 0.053$ & $0.645 \pm 0.060$ & $0.656 \pm 0.049$ & $0.662 \pm 0.060$ \\
    Random      & $0.529 \pm 0.043$ & $0.676 \pm 0.130$ & $0.739 \pm 0.105$ & $0.717 \pm 0.111$ & $0.772 \pm 0.041$ \\
    \bottomrule
    \end{tabular}
\end{table}

The key takeaway of these experiments is that $\FWSwap$ consistently yields the best performing results on the low budget.
Furthermore, we also note greater stability of our results: across all experiments, the variance of the performance of $\FWSwap$ is lower than those of other methods.

\section{Limitations}
We highlight three limitations of the proposed approach that open perspective for future work.
First, as with other representation-based active set methods, our approach relies on the quality of the learned representation features (see discussion in \cref{apx:simclr-vs-dino}).
Second, the underlying cost matrix and corresponding Sinkhorn kernel can be computationally expensive for large datasets, which would require further research on the computational OT side to scale.
Finally, while $\FWSwap$ is competitive at a small budget, its performance becomes less consistent when the budget increases (see discussion in \cref{apx:additionalexperiments}).


\section{Conclusion}

We propose a theoretically grounded approach to low-budget active learning, combining a principled criterion, the Sinkhorn divergence, with an algorithm designed for the combinatorial structure of the problem. Our algorithm exploits the convexity and differentiability of the criterion in a combinatorial local search, with quantitative suboptimality bounds and statistical guarantees.
Empirically, $\FWSwap$ achieves superior accuracy in the low-budget regime, with the advantage most pronounced at very small budgets where representativeness matters most.
In particular, the results show that our method can be highly suitable to medical applications as EEG signals and radio imaging.

\subsection*{Acknowledgments}

This research was supported in part by the French National Research Agency (ANR) in the framework of the PEPR IA FOUNDRY project (ANR-23-PEIA-0003)
and through the MIAI Cluster @ Grenoble (reference ANR-23-IACL-0006).
We want to thank Gabriel Peyré for valuable discussions on degeneracy in entropic optimal transport.

\subsection*{Generative AI use statement}
In this work, we used generative AI tools to provide preliminary results for mathematical proofs, which were subsequently refined by the authors,
to assist in the generation of some code for the computational experiments, and to improve clarity of the text.

\bibliography{bibliographie}
\bibliographystyle{apalike}

\appendix

\section*{Appendix}

The appendix complements the main paper with proofs of the results, as well as auxiliary technical developments and additional computational results. 
Note that the sections are ordered by technical difficulty rather than their order as presented in the main paper. More precisely, we detail:
\begin{itemize}[leftmargin=0.5cm]
    \item in \cref{apx:additionalproofs}, the \textbf{statements} and proofs of the two results mentioned in \cref{sec:framework},
    \item in \cref{apx:runtime}, the runtime analysis of the $\FWSwap$ algorithm.
    \item in \cref{apx:gen_bound}, the proof and extension of the \textbf{generalization bound} with dimension-free guarantees stated in \cref{sec:genbound},
    \item in \cref{apx:entropicbiasstudy}, the \textbf{analysis of Sinkhorn bias} provided of cluster-central-outlier example,
    \item in \cref{apx:vertexswap}, the proof of \textbf{theoretical guarantees} of our algorithm,
    \item in \cref{apx:simclr-vs-dino}, a comparison based on the \textbf{feature space} defined by the SSL model,
    \item in \cref{apx:additionalexperiments}, \textbf{experimental results} complementary to the empirical observations in \cref{sec:emp-results}.
\end{itemize}



\section{Auxiliary results}\label{apx:additionalproofs}

In this section, we prove two auxiliary results on reformulating previous active learning methods in an OT framework, and on the triviality of the convex relaxation of the selection problem, regardless of the divergence.

\begin{proposition}[Previous methods within a OT framework]\label{prop:probcoverOT}
The optimization problems underlying  the ProbCover and k-medoid methods
can be cast as minimizing a Wasserstein distance from the dataset to a measure supported on $\budget$ points, with k-medoids equipped with the Euclidean distance and ProbCover equipped with the ground cost:
\begin{align*}
c^{\mathrm{PC}}(\vy_1, \vy_2) = \begin{cases}
0 & \text{if }\, \|\vy_1 - \vy_2\|_2 \leq \delta \\
1 & \text{otherwise.}
\end{cases}
\end{align*}
\end{proposition}
\begin{proof}
We show the criterion optimized by both methods can, for any \selection, be expressed as the value of an OT problem in which the selected marginal is free over the points in the support of \selection.

The k-medoids criterion for a selection $\selection \in \HypersimplexVertices$
is $\sum_{j \in [\npoints]} \min_{i \in \mathrm{supp}(\selection)} C_{ij}$
which assigns each dataset point to its nearest medoid at cost $C_{ij}$.
This is exactly the minimum Wasserstein distance under a free second marginal
\begin{align*}
\min_{\substack{\vm \in \Simplex{\npoints} \\ \vm_i \leq \selection_i}} W_C(\sum_{i\in[\npoints]} \vm_i \delta_{\vy_i}, \EmpiricalDistDataset).
\end{align*}

Since the optimal transport plan assigns the full mass of each point to its nearest medoid,
minimizing over $\selection\in\HypersimplexVertices$ recovers k-medoids exactly.

A similar formulation can be proposed for ProbCover, for which the underlying optimization problem can be expressed as
\begin{align*}
\min_{\selection \in \HypersimplexVertices}
\frac{1}{\npoints} \sum_{i \in [\npoints]} \mathbf{1}\Big( \min_{j:\,\selection_j=1} \|\vy_i - \vy_j\|_2 > \delta \Big).
\end{align*}
Consider the optimal transport problem
\begin{align*}
W_{c^{\mathrm{PC}}}(\mu, \EmpiricalDistDataset)
= \min_{\TransportPlan \in \Pi(\mu, \EmpiricalDistDataset)}
\sum_{(i,j) \in [\npoints] \times [\npoints]} \TransportPlan_{ij} \, c^{\mathrm{PC}}(\vy_i, \vy_j),
\end{align*}
where $\Pi(\mu, \EmpiricalDistDataset)$ is the set of couplings with marginals $\mu$ and $\EmpiricalDistDataset$.
Because the cost $c^{\mathrm{PC}}$ is binary, any mass transported between points $i$ and $j$ with $\|\vy_i - \vy_j\|_2 \leq \delta$ incurs zero cost, and any other assignment incurs a unit cost.\\
For each $\vy_i$, if there exists a selected point $\vy_j$ such that $\|\vy_i - \vy_j\|_2 \leq \delta$, then assigning its full mass $1 / \npoints$ to such a point yields zero cost.
Furthermore, a selected center can be attributed a mass $k /\npoints$, with $k$ the number of points it captures.
All points outside the circle around any of the centers needs to be transported to one of the centers arbitrarily, inducing a constant cost of one.
Hence, the optimal transport cost satisfies for any $\selection \in \HypersimplexVertices$
\begin{align*}
\min_{\substack{\vm \in \Simplex{\npoints} \\ \vm_i \leq \selection_i}} W_{c^{\mathrm{PC}}}(\sum_{i \in [\npoints]} \delta_{\vy_i} \vu_i, \EmpiricalDistDataset )
= \frac{1}{\npoints} \sum_{i=1}^\npoints \mathbf{1}\Big( \min_{j:\,\selection_j=1} \|\vy_i - \vy_j\|_2 > \delta \Big),
\end{align*}
recovering the ProbCover loss.
The problems underlying both methods can be expressed as the minimization of a Wasserstein distance from a distribution supported on the coreset to \EmpiricalDistDataset.
\end{proof}

\begin{proposition}[Convex relaxation over the hypersimplex yields a trivial bound]\label{prop:convexrelax}
For any divergence $D$, the convex relaxation of the coreset minimization problem~\eqref{eq:genericlearning} is the following trivial optimization problem over the hypersimplex $\Hypersimplex$
\begin{align*}
    \min_{\selection \in \Hypersimplex} D(\EmpiricalDist_{\selection}, \EmpiricalDistDataset) = 0,
\end{align*}
which attains its minimum at $\selection = \frac{\budget}{\npoints}\1$ and is independent of the data.
\end{proposition}
\begin{proof}
By symmetry, the point $\frac{\budget}{\npoints}\1$ is the center of $\Hypersimplex$.
The corresponding probability distribution, once rescaled, is exactly $\EmpiricalDistDataset$.
Since $D$ is a divergence, the unique optimum over $\Hypersimplex$, which is a convex combination of all vertices of $\Hypersimplex$ with equal weights, trivially achieves an objective value of zero.
\end{proof}

\section{Number of swap iterations analysis}\label{apx:runtime}
We now focus on the number of iterations of the swap algorithm, with three key observations noted O1, O2, O3.

Note $r(\selection)$ the gradient of the divergence at $\selection$. Swap operated a deterministic map from one vertex to another of the hypersimplex so the swap chain first forms a tail, and then a cycle. Let $\gamma(\selection)$ be the swap gap at the current vertex:
\begin{align*}
    \gamma(\selection) = - \min_{\vvv \in \swapset(\selection)} \innp{r(\selection^c)}{\vvv - \selection^c} > 0.
\end{align*}
Positivity comes from no vertices being optimal for the continuous problem.

\begin{enumerate}
    \renewcommand{\labelenumi}{\textbf{O\arabic{enumi}.}}
    \item Average gap on the cycle is bounded above. Let $T_c$ be the length of the cycle. We have, with $L$ the Lipschitz constant, along iterations of the cycle:
    \begin{align*}
        f_\lambda(\selection^{(c+1)}) \le f_\lambda(\selection^c) - \gamma(\selection^c) + \frac{L}{2}.
    \end{align*}
    Telescoping, $0 < \frac{1}{T_c\left(\sum_{c=1}^{T_c}\gamma(\selection^c)\right)} \le \frac{L}{2}$. A cycle can only occur once the swap gap has reduced to the scale of the curvature constant $\frac{L}{2}$, so, where the objective function is already ``flat''.
    \item The length of the tail $T_t$ is bounded by the range $R$ (max-min function value on the vertices). Telescope the descent inequality on the tail: 
    \begin{align*}
        \sum_{c=1}^{T_t} \left(\gamma(\selection^c) - \frac{L}{2}\right) \le f_\lambda(\selection^0) - f_\lambda(\selection^{T_t}).
    \end{align*}
    Defining $\gamma_{\mathrm{\min}}$ as the minimum gap on the whole tail, whenever $\gamma_{\mathrm{\min}} > \frac{L}{2}$ we have the inequality $T_t \le \frac{R}{\gamma_{\mathrm{\min}} - \frac{L}{2}}$. 
    This bounds the tail with the function values of the Sinkhorn divergence.
    \item The core is fully acquired in at most $\budget_0 < \budget$ iterations, after which the problem is reduced to a hypersimplex of parameters $\npoints - \budget_0$, $\budget - \budget_0$, on which O1 and O2 apply. This is formalized in \Cref{ass:core}.
\end{enumerate}

\core*

\Cref{ass:core} can be verified empirically. $S_0$ is a set of points that are unambiguously central for a cluster in the data, and $S_0 = \emptyset$ means several candidates are close to the optimal for each cluster. 
We report in \cref{tab:cycle-results} the number of iterations of $\FWSwap$ over multiple starts on STL-10. We report separately the two parts of the trajectory: the tail (vertices visited once) from the cycle. We observe a tail roughly proportional to the budget on average, while the cycle length remains small and constant. 

\begin{table}[hbt!]
\tiny
  \caption{Details of $\FWSwap$ trajectory for different budgets.}\label{tab:cycle-results}
  \centering
    \begin{tabular}{lccccc}
    \toprule
    \textbf{Budget} & \textbf{End Cycle} & \textbf{Start Cycle} & \textbf{Detected Cycle} & \textbf{Detected Std} & \textbf{Cycle Length} \\
    \midrule
    10 & $11.0$ & $12.0$ & $14.0$ & $1.22474$ & $3.0$ \\
    20 & $22.6$ & $23.6$ & $25.6$ & $2.19089$ & $3.0$ \\
    30 & $33.2$ & $34.2$ & $36.2$ & $1.30384$ & $3.0$ \\
    40 & $44.8$ & $45.8$ & $47.8$ & $2.68328$ & $3.0$ \\
    50 & $54.8$ & $55.8$ & $57.8$ & $2.77489$ & $3.0$ \\
    60 & $64.2$ & $65.2$ & $67.2$ & $4.91935$ & $3.0$ \\
    \bottomrule
    \end{tabular}
\end{table}

\cref{tab:cycle-results} shows a long tail and a short cycle, which coincides with what happens under a large core assumption. We ran additional experiments to test \textbf{O3} and we could confirm that the length of the tail is linked primarily to the acquisition of the core $S_0$.

\begin{table}[hbt!]
\tiny
  \caption{Number of acquired points and initial points retained in the final selection.}\label{tab:acquired-initial}
  \centering
    \begin{tabular}{lcc}
    \toprule
    \textbf{Budget} & \textbf{Acquired + Initial} & \textbf{Percentage} \\
    \midrule
    10 & $(1+8=9)$  & $90\%$ \\
    20 & $(2+16=18)$ & $90\%$ \\
    30 & $(5+24=29)$ & $96.67\%$ \\
    40 & $(8+31=39)$ & $97.5\%$ \\
    50 & $(8+38=46)$ & $92\%$ \\
    60 & $(9+48=57)$ & $95\%$ \\
    \bottomrule
    \end{tabular}
\end{table}

\cref{tab:acquired-initial} reports the number of acquired-once points during the swap and the number of original points from the initialization that have not been swapped, which together form a core. 
On all budgets, the core accounts for more than $90\%$ of the final selection, serving as empirical evidence for \cref{ass:core}.

\section{Generalization bound}\label{apx:gen_bound}

We detail here the different results supporting \cref{thm:genbound}.
We make the following assumptions:
\begin{assumption}\label{assum:unitsupport}
The data are i.i.d.~samples from a distribution $\mu$ such that for $\mu$-almost every $\vy$, we have $\norm{\vy}_2 \leq \frac12$.
\end{assumption}

Our analysis is based on the following sample complexity result for the estimation of entropic OT from samples.

\begin{theorem}[Strong duality, Theorem 1 of~\cite{rigollet2025sample}]\label{thm:strongduality}
Let $\alpha,\beta$ be probability measures with bounded support on $\R^d$ and let $\varepsilon>0$.
For bounded measurable $f,g:\R^d\to\R$ set the dual objective
\begin{align}\label{eq:generaldualobj}
\Phi_{f,g}(\alpha,\beta) \;=\; \innp{f}{\alpha} + \innp{g}{\beta}
  \;-\; \varepsilon \innp{\alpha \otimes \beta}{e^{(f\oplus g-c)/\varepsilon}-1},
\end{align}
where $(f \oplus g)(\vy,\vz) = f(\vy)+g(\vz)$ (we define $\otimes$ similarly).
Then the dual entropic OT problem is
\begin{align}\label{eq:strongdual}
\SinkObjEps(\alpha,\beta) \;=\; \sup_{(f,g)\in L^\infty(\alpha)\times L^\infty(\beta)} \Phi_{f,g}(\alpha,\beta),
\end{align}
and the supremum is attained at a pair $(f,g)$, unique up to $(f,g)\mapsto(f+a,g-a)$, $a\in\R$.
Moreover, primal and dual solutions are linked as follows: for a pair $(f,g)$, let $\pi$ be the
measure with density $e^{(f\oplus g-c)/\varepsilon}$ with respect to $\alpha\otimes\beta$. Then
$(f,g)$ is optimal for~\eqref{eq:strongdual} if and only if $\pi\in\Pi(\alpha,\beta)$, in which
case $\pi$ is the optimal plan.
\end{theorem}

\begin{theorem}[Theorem 3,~\cite{rigollet2025sample}]\label{thm:samplecomplexsink}
Let $\mu$ and $\nu$ be two distributions satisfying \cref{assum:unitsupport}.
Let $\SinkObjEps(\mu_n,\nu_n)$ be the entropic optimal transport cost between two
empirical distributions with $n$ mutually independent samples from $\mu$ and $\nu$ respectively.
Then, there exists $\kappa_0(\varepsilon)>0$ depending (exponentially) only on the regularization parameter $\varepsilon$
and such that
\begin{align}\label{eq:samplecomplexsink}
|\E[\SinkObjEps(\hat{\mu}_n,\hat{\nu}_n)] - \SinkObjEps(\mu,\nu)| \leq \frac{\kappa_0(\varepsilon)}{n}.
\end{align}
\end{theorem}

As a direct consequence, we can establish a similar result for the Sinkhorn divergence $\SinkDivEps(\mu, \hat{\mu}_n)$. 

\begin{lemma}[Sample complexity of the Sinkhorn divergence]\label{lemma:samplecomplexdiv}
Let $\mu$ satisfy \cref{assum:unitsupport} and let $\hat\mu_n$ be the empirical measure of
$n$ i.i.d.~samples from $\mu$. Set $M=\sup_{\vy,\vz\in\mathrm{supp}(\mu)}c(\vy,\vz)$,
so that $M\leq 1$ under \cref{assum:unitsupport}. Then
\begin{align}\label{eq:lemmabound}
  \E\big[\SinkDivEps(\mu,\hat\mu_n)\big]
  \;\leq\;\frac1n\Big(\kappa_0(\varepsilon)+\tfrac{\varepsilon}{2}\big(e^{M/\varepsilon}-1\big)\Big)
  \;\leq\;\frac{\kappa(\varepsilon)}{n},
\end{align}
with $\kappa(\varepsilon)\;:=\;\kappa_0(\varepsilon)+\tfrac{\varepsilon}{2}\big(e^{1/\varepsilon}-1\big)$,
and $\kappa_0(\varepsilon)$ the constant of \cref{thm:samplecomplexsink}.
In particular, $\kappa$ depends only (exponentially) on $1/\varepsilon$.
\end{lemma}

\begin{proof}
By the definition of the Sinkhorn divergence,
\begin{align}\label{eq:decompdiv}
  \E\big[\SinkDivEps(\mu,\hat\mu_n)\big]
  =\E\big[\SinkObjEps(\mu,\hat\mu_n)\big]
  -\tfrac12 \SinkObjEps(\mu,\mu)
  -\tfrac12\,\E\big[\SinkObjEps(\hat\mu_n,\hat\mu_n)\big].
\end{align}
We bound the first term from above and the third from below.
Throughout, $0\leq \SinkObjEps(\alpha,\beta)\leq M$ for $\alpha,\beta$ supported in
$\mathrm{supp}(\mu)$: the lower bound because $c\geq0$ and $\KL\geq0$, the upper one by
evaluating the primal objective with the product plan $\pi = \alpha\otimes\beta$.
Every $\SinkObjEps$ that depends on the sample has an empirical measure in at least one argument,
$\Phi_{f, g}$ in~\eqref{eq:strongdual} is then a continuous function of the positions of the atoms of this empirical measure.

We first bound the cross term.
Let $\beta$ be any probability measure supported on $\mathrm{supp}(\mu)$ and let $(f,g)$ be an
optimal pair of potentials for $\SinkObjEps(\mu,\beta)$ from \cref{thm:strongduality}.
Integrating the last term of~\eqref{eq:generaldualobj} in the second argument $\vz$ first,
\begin{align*}
\Phi_{f,g}(\alpha,\beta) & = \innp{h}{\alpha} + \innp{g}{\beta}, \\
h(\vy) & := f(\vy) - \varepsilon \int \Big( e^{(f(\vy)+g(\vz)-c(\vy,\vz))/\varepsilon} - 1 \Big) \mathrm{d}\beta(\vz),
\end{align*}
so that $\alpha \mapsto \Phi_{f,g}(\alpha,\beta)$ is affine.
Let $\hat\mu_n'$ be the empirical measure of a sample $X_1',\dots,X_n'\sim\mu$ and
let $\E'$ denote the expectation over it. Since
$\E'\innp{h}{\hat\mu_n'} = \innp{h}{\mu}$,
we have
\begin{align*}
\SinkObjEps(\mu,\beta) = \Phi_{f,g}(\mu,\beta)
  = \E'\big[\Phi_{f,g}(\hat\mu_n',\beta)\big]
  \;\leq\; \E'\big[\SinkObjEps(\hat\mu_n',\beta)\big],
\end{align*}
the first equality by optimality of $(f,g)$, the second by affineness, and the inequality because
$f$ and $g$ are bounded, hence feasible solutions to~\eqref{eq:strongdual} for the pair
$(\hat\mu_n',\beta)$.
Running the same argument on a mixture $\alpha=(1-t)\alpha_0+t\alpha_1$ shows that
$\alpha\mapsto\SinkObjEps(\alpha,\beta)$ is convex, and the inequality above corresponds to Jensen's.
Applying it with $\beta=\hat\mu_n$, which is independent of $\hat\mu_n'$, and averaging,
\begin{align}\label{eq:crossbound}
  \E\big[\SinkObjEps(\mu,\hat\mu_n)\big]
  \;\leq\;\E\big[\SinkObjEps(\hat\mu_n',\hat\mu_n)\big]
  \;\leq\; \SinkObjEps(\mu,\mu)+\frac{\kappa_0(\varepsilon)}{n},
\end{align}
the last inequality being \cref{thm:samplecomplexsink} applied with $\nu=\mu$: the two samples
are i.i.d.~from $\mu$, of equal size $n$, and mutually independent as required.

We can now bound the self-transport term.
Let $(f,g)$ be an optimal pair for $\SinkObjEps(\mu,\mu)$, again given by
\cref{thm:strongduality}; it is deterministic. Write $\Psi := e^{(f\oplus g-c)/\varepsilon}-1$.
By \cref{thm:strongduality} the measure with density $\Psi+1$ with respect to $\mu\otimes\mu$ has
first marginal $\mu$, that is, its density integrates to one in the second variable:
\begin{align}\label{eq:schrodinger}
\int e^{(f(\vy)+g(\vz)-c(\vy,\vz))/\varepsilon}\,\mathrm{d}\mu(\vz) = 1
\qquad\text{for $\mu$-a.e.\ }\vy .
\end{align}
Integrating~\eqref{eq:schrodinger} against $\mu$ gives $\innp{\mu\otimes\mu}{\Psi}=0$, hence
\begin{align}\label{eq:Smumu}
\SinkObjEps(\mu,\mu) = \Phi_{f,g}(\mu,\mu) = \innp{f+g}{\mu}.
\end{align}
Since $f$ and $g$ are bounded, they are admissible in~\eqref{eq:strongdual} for the pair
$(\hat\mu_n,\hat\mu_n)$, so
\begin{align*}
\SinkObjEps(\hat\mu_n,\hat\mu_n) \;\geq\; \Phi_{f,g}(\hat\mu_n,\hat\mu_n)
= \innp{f+g}{\hat\mu_n} - \varepsilon \innp{\hat\mu_n\otimes\hat\mu_n}{\Psi}.
\end{align*}
Since $X_i\sim\mu$ we have $\E\innp{f+g}{\hat\mu_n}=\innp{f+g}{\mu}$.
Separating diagonal from off-diagonal pairs,
\begin{align*}
  \E\innp{\hat\mu_n\otimes\hat\mu_n}{\Psi} & = \frac{1}{n^2}\sum_{i\neq j}\E\big[\Psi(X_i,X_j)\big]
   +\frac{1}{n^2}\sum_{i}\E\big[\Psi(X_i,X_i)\big] \\
  & =\Big(1-\tfrac1n\Big)\underbrace{\innp{\mu\otimes\mu}{\Psi}}_{=0}
   +\frac1n\Big(\E_\mu\big[e^{(f+g)/\varepsilon}\big]-1\Big),
\end{align*}
where we used that $X_i$ and $X_j$ are independent for $i\neq j$, and
$\Psi(\vy,\vy)=e^{(f(\vy)+g(\vy))/\varepsilon}-1$ since $c(\vy,\vy)=0$. Combining with~\eqref{eq:Smumu},
\begin{align}\label{eq:selfbound}
  \E\big[\SinkObjEps(\hat\mu_n,\hat\mu_n)\big]
  \;\geq\; \SinkObjEps(\mu,\mu)-\frac{\varepsilon}{n}\Big(\E_\mu\big[e^{(f+g)/\varepsilon}\big]-1\Big).
\end{align}

We can now bound the second quantity of the right-hand side of~\eqref{eq:selfbound}.
Set $I_g=\int e^{g/\varepsilon}\mathrm{d}\mu\in(0,\infty)$.
By~\eqref{eq:schrodinger} and $c\leq M$ on $\mathrm{supp}(\mu)\times\mathrm{supp}(\mu)$,
\begin{align*}
& 1 = \int e^{(f(\vy)+g(\vz)-c(\vy,\vz))/\varepsilon}\mathrm{d}\mu(\vz)
  \;\geq\; e^{-M/\varepsilon}\,e^{f(\vy)/\varepsilon} I_g , \\
  & e^{f(\vy)/\varepsilon} \leq\;\frac{e^{M/\varepsilon}}{I_g} \qquad\text{for $\mu$-a.e.\ }\vy 
\end{align*}
Integrating the remaining factor,
\begin{align}\label{eq:diagbound}
\E_\mu\big[e^{(f+g)/\varepsilon}\big]
= \int e^{f(\vy)/\varepsilon}\,e^{g(\vy)/\varepsilon}\,\mathrm{d}\mu(\vy)
\;\leq\; \frac{e^{M/\varepsilon}}{I_g}\int e^{g/\varepsilon}\mathrm{d}\mu
\;=\; e^{M/\varepsilon},
\end{align}
and therefore $\E_\mu[e^{(f+g)/\varepsilon}]-1 \leq e^{M/\varepsilon}-1$.

Substituting~\eqref{eq:crossbound},~\eqref{eq:selfbound} and~\eqref{eq:diagbound} into~\eqref{eq:decompdiv}, the three terms $\SinkObjEps(\mu,\mu)$ cancel and
\begin{align*}
  \E\big[\SinkDivEps(\mu,\hat\mu_n)\big]
  \;\leq\;\frac{\kappa_0(\varepsilon)}{n}+\frac{\varepsilon}{2n}\big(e^{M/\varepsilon}-1\big),
\end{align*}
which is~\eqref{eq:lemmabound}; the final bound follows from $M\leq1$.
\end{proof}

\begin{lemma}\label{lemma:selfotbounds}
The entropic self-transport is bounded by two times the variance: $\SinkObjEps(\mu,\mu) \leq 2\Var[\mu]$. Moreover if $\mu$ satisfying \cref{assum:unitsupport}, we have 
\begin{align*}
 \SinkObjEps(\mu,\mu) \leq 2\Var[\mu]\leq 1.
\end{align*}
\end{lemma}
\begin{proof}
The entropic self-transport cost is bounded by the cost of any valid coupling;
we consider in particular the product $\mu \otimes \mu$
\begin{align*}
\SinkObjEps(\mu,\mu) \leq \innp{C}{\mu \otimes \mu} + \KL(\mu \otimes \mu, \mu \otimes \mu).
\end{align*}
The $\KL$ divergence is zero when evaluated on the same two measures and the product coupling gives
\begin{align*}
\innp{C}{\mu \otimes \mu} = \int \int \norm{\vz-\vy}_2^2 d\mu(\vz) d\mu(\vy) = 2\Var[\mu].
\end{align*}
By assumption, we have $\norm{\vz-\vy}_2^2 \leq 1$ pointwise almost everywhere, yielding the upper bound of one.
\end{proof}

We can now restate and prove our main generalization theorem.

\genbound*

\begin{proof}
For any two measures $\alpha,\beta$ respecting \cref{assum:unitsupport}, let $\TransportPlan$ be the optimal entropic coupling minimizing $\SinkObjEps(\alpha,\beta)$,
we can derive
\begin{subequations}
\begin{align}
\int f_{\learningparam} d\alpha - \int f_{\learningparam} d\beta &= \int (f_{\learningparam}(\vz) - f_{\learningparam}(\vy)) d\TransportPlan(\vz,\vy) \label{eq:genericcouplingineq1} \\
& \leq L  \int \norm{\vz-\vy}_2 d\TransportPlan(\vz,\vy) \label{eq:genericcouplingineq2} \\
& \leq L \sqrt{\int \norm{\vz-\vy}_2^2 d\TransportPlan(\vz,\vy) } \label{eq:genericcouplingineq3} \\
& \leq L \sqrt{\SinkObjEps(\alpha,\beta)},\label{eq:genericcouplingineq4} \\
& = L \sqrt{\SinkDivEps(\alpha,\beta) + \frac12 \SinkObjEps(\alpha,\alpha) + \frac12 \SinkObjEps(\beta,\beta)} \label{eq:genericcouplingineq5}\\
& \leq L \sqrt{\SinkDivEps(\alpha,\beta) + \Var[\alpha] + \Var[\beta] } \label{eq:genericcouplingineq6}
\end{align}
\end{subequations}
where the first equality~\eqref{eq:genericcouplingineq1} uses that $\TransportPlan$ is a coupling,
~\eqref{eq:genericcouplingineq2} uses the Lipschitz bound,~\eqref{eq:genericcouplingineq3} applies Cauchy-Schwarz and
~\eqref{eq:genericcouplingineq4} uses the nonnegativity of $\KL(\cdot, \alpha \otimes \beta)$,
~\eqref{eq:genericcouplingineq5} uses the definition of the divergence and~\eqref{eq:genericcouplingineq6} uses \cref{lemma:selfotbounds}.
Applying this inequality with $\hat{\mu}_{S^*}$ and $\EmpiricalDistDataset$, we get the first bound on the empirical risk and the coreset risk. 
Applying the inequality a second time with $\mu$ and $\EmpiricalDistDataset$ gives
\begin{align*}
R(\learningparam) - \hat{R}_n(\learningparam) & \leq L \sqrt{ \SinkDivEps(\mu,\EmpiricalDistDataset) + \Var[\mu] + \Var[\EmpiricalDistDataset] } \\
\E[R(\learningparam) - \hat{R}_n(\learningparam)] & \leq L \sqrt{ \frac{\kappa(\varepsilon)}{\npoints} + \frac{2\npoints-1}{\npoints} \Var[\mu] },
\end{align*}
where the second line takes the expectation over possible realizations of $\EmpiricalDistDataset$, applies \cref{lemma:samplecomplexdiv} and Jensen's inequality to $\sqrt{\cdot}$,
along with the expectation of the empirical variance.
\end{proof}

\section{Sinkhorn bias for Cluster-Central-Outlier instances}\label{apx:entropicbiasstudy}

In this section, we analyze the family of instances defined in \cref{sec:sinkbiasal}, and quantify the bias of $S_\varepsilon$ on the three configurations presented there.

We start with closed forms for Wasserstein costs of the three choices, thanks to the simplicity of one-dimensional optimal transport.
\begin{proposition}\label{prop:wasscostunivariate}
The Wasserstein cost of the three possible selections are:
\begin{align*}
W(\alpha_C, \EmpiricalDistDataset) = \frac{2m \Delta^2}{2m+1}, \quad
W(\alpha_L, \EmpiricalDistDataset) = \frac{\Delta^2}{2}, \quad\text{and}\quad
W(\alpha_E, \EmpiricalDistDataset) = \frac{\Delta^2}{2m+1}
\end{align*}
with $\alpha_E$ being the configuration with minimum cost.
\end{proposition}
\begin{proof}
To transport $\alpha_C$ to $\EmpiricalDistDataset$, the same mass must be transported to both the left and right points.
The reverse is performed to transport $\alpha_E$ to $\EmpiricalDistDataset$, with $\frac{1}{2} - \frac{m}{2m+1}$ being the excess mass at each extreme point that has to be transported to the origin. To transport $\alpha_L$, the left point has an excess of mass that needs to be transported of $\frac12 - \frac{m}{2m+1} = \frac{1}{2(2m+1)}$.
Because the ground cost is the Euclidean squared distance, the optimal plan transports this excess to the origin, and displaces to the right point enough mass from the origin to match its weight in $\EmpiricalDistDataset$
which is $\frac{m}{2m+1}$.
The total transport cost is $\Delta^2 \left( \frac{1}{2(2m+1)} + \frac{m}{2m+1} \right) = \frac{\Delta^2}{2}$.
\end{proof}

The ranking of the configurations of Sinkhorn costs is identical for low values of $\varepsilon$, but we show in the following proposition that it inverts for $\varepsilon$ above a certain threshold.

\begin{proposition}
The limit of the Sinkhorn objective values of the three configurations at $\varepsilon \rightarrow \infty$:
\begin{align*}
\mathcal{S}_{\infty}(\alpha_C, \EmpiricalDistDataset) = \frac{2m \Delta^2}{2m+1}, \quad
\mathcal{S}_{\infty}(\alpha_L, \EmpiricalDistDataset) = \frac{(6m+1)\Delta^2}{2 (2m+1)}, \quad
\mathcal{S}_{\infty}(\alpha_E, \EmpiricalDistDataset) = \frac{(4m+1)\Delta^2}{2m+1}.
\end{align*}
There exists a finite $\varepsilon$ for which $\alpha_C$ results in the smallest loss.
\end{proposition}
\begin{proof}
The Sinkhorn cost of a transport plan is the sum of the inner product with the distance matrix plus the regularization term, $\varepsilon$ multiplied by the KL divergence.
Driving $\varepsilon \rightarrow \infty$ forces the KL divergence to zero, meaning the transport plan becomes the outer product of the two input weight vectors,
and the cost becomes $\sum_{i\in [n_{\alpha}], j \in \{1,2,3\} } (x_i - y_j)^2 \alpha_i \EmpiricalDistDataset_j$, with $x_i$ the $i$-th point of the $\alpha$ distribution and $y_j$ the $j$-th point of the dataset.
We can compute the corresponding expression for each of the configurations.
Since the Sinkhorn objective is continuous with respect to $\varepsilon$ and $\mathcal{S}_{+\infty}(\alpha_C, \EmpiricalDistDataset)$ is strictly smaller than the loss of the two other configurations,
there exists a finite $\bar\varepsilon$ above which the central configuration is always preferred.
\end{proof}

We highlight here that the inherent bias of the Sinkhorn cost also produces a paradox for the active learning application:
it can lead to selecting \emph{fewer} points to train on than the allocated budget, which in general will go against the expected and empirical behavior that one should train on as many points as the budget allows. 

\begin{lemma}\label{lemma:sinklossexpralphaE}
The entropy losses of the extreme configuration can be expressed as
\begin{align}\label{eq:sinklossexpralphaE}
\SinkObjEps(\alpha_E,\EmpiricalDistDataset) = W(\alpha_E,\EmpiricalDistDataset) + \frac{2m}{2m+1} \SinkObjEps(\alpha_E,\alpha_E).
\end{align}
\end{lemma}
\begin{proof}
The coupling $\TransportPlan$ is a $2\times 3$ matrix with rows corresponding to the mass supported on $\{-\Delta,\Delta\}$
and columns corresponding to the mass supported on $\{-\Delta,0,\Delta\}$, with marginals $(\frac12,\frac12)$ and $(\frac{m}{2m+1}, \frac{1}{2m+1}, \frac{m}{2m+1})$.
Both marginals and the ground costs are symmetric around the origin, the optimal plan is also symmetric and can be expressed as:
\begin{align*}
\TransportPlan^* = \begin{bmatrix*}
   a & b & d \\
   d & b & a
\end{bmatrix*}.
\end{align*}
The marginal constraint of the central column results in $2b = \frac{1}{2m+1} \Rightarrow b = \frac{1}{4m+2}$.
This value corresponds to the product of the probabilities of the marginals for the whole column, the corresponding term in the KL divergence is thus zero.
The contribution to the transport cost is the constant $2b\Delta^2 = \Delta^2 / (2m+1)$.

Using $p = \frac{m}{2m+1}$ for convenience and setting $d = p - a$, joint optimization over the whole transport plan reduces to determining the optimal $a$.
The weights of the marginal $\EmpiricalDistDataset$ on the remaining $2\times 2$ block are $(p, p)$, which is proportional to $\alpha_E$.
Using $\KL(p \mu, \nu) = \lambda \KL(p \mu, \nu)$ for $p > 0$, we can reformulate the objective as:
\begin{align*}
F(a) & = (1-2p) \Delta^2 + 2p\left[4\Delta^2\cdot\tfrac{d}{2p} + \varepsilon\,\KL(\tilde\TransportPlan,\alpha_E\otimes\alpha_E)\right] = (1-2p)\Delta^2 + 2p\cdot G\!\left(\tfrac{d}{2p}\right) \\
\text{where } G(s) &= 4\Delta^2 s + \varepsilon\,\KL(\tilde\TransportPlan(s),\alpha_E \otimes \alpha_E).
\end{align*}
$G(s)$ is exactly the Sinkhorn objective for $\SinkObjEps(\alpha_E,\alpha_E)$ parameterized by its off-diagonal entry $s$.
Combining with the expression of the Wasserstein cost in \cref{prop:wasscostunivariate}, we can identify the expression~\eqref{eq:sinklossexpralphaE} and conclude.
\end{proof}

Finally, we show that the Sinkhorn divergence, in contrast, avoids this paradox by removing the centrality bias of the Sinkhorn objective.

\begin{proposition}
The Sinkhorn divergence $D_\varepsilon$ leads to the extreme configuration being preferred for any value of $\varepsilon$.
\end{proposition}
\begin{proof}
We first derive the expression of the symmetric optimal transport cost for any $\varepsilon$.
The measure $\alpha_C$ has a single atom, the transport cost and the KL divergence are both zero, $\SinkObjEps(\alpha_C, \alpha_C) = 0 \, \forall \varepsilon$.
The measures $\alpha_L$ and $\alpha_E$ are both supported on two atoms with uniform weights $(1/2, 1/2)$, the analysis is similar for both.

Defining the two-point measure as $1/2(\delta_a + \delta_b)$, setting $c = (a-b)^2$ and noting that the transport plan $\TransportPlan$ is symmetric,
it is fully parameterized by a single scalar $t$ such that $\TransportPlan_{aa} = \TransportPlan_{bb} = \frac12 - t$, $\TransportPlan_{ab} = \TransportPlan_{ba} = t$.
The expression $\innp{C}{\TransportPlan}$ reduces to $2tc$.
The product measure has weight $1/4$ on all entries, the KL divergence becomes
\begin{align*}
2(\frac12 - t)\log\left(\frac{\frac12 - t}{\frac14} \right) + 2t\log\left(\frac{t}{\frac14}\right)
\end{align*}
and we can derive the complete expression as a function of $t$, and the optimal $t^*$ from first-order optimality:
\begin{align*}
& \SinkObjEps(\alpha,\alpha) = \phi(t) = 2tc + \varepsilon \left((1-2t)\log(2-4t) + t2 \log(4t) \right)\\
& t^* = \frac{1}{2(1 + e^{c/\varepsilon})} \\
& \phi(t^*) = \frac{c}{1 + e^{c/\varepsilon}} + \frac{c e^{c/\varepsilon}}{1 + e^{c/\varepsilon}} + \varepsilon \log \frac{2}{1+ e^{c/\varepsilon}} \\
& \phi(t^*) = c + \varepsilon \log\left( \frac{2}{1 + e^{c/\varepsilon}} \right) = \varepsilon \log\left( \frac{2}{1 + e^{-c/\varepsilon}} \right).
\end{align*}
Specializing for $\alpha=\alpha_L$ with $c=\Delta^2$ and $\alpha_E$ with $c=4\Delta^2$, we obtain:
\begin{align}
& \SinkObjEps(\alpha_L,\alpha_L) = \varepsilon \log \frac{2}{1 + e^{-\Delta^2 / \varepsilon}} \label{eq:sinkSelfLossLL} \\
& \SinkObjEps(\alpha_E,\alpha_E) = \varepsilon \log \frac{2}{1 + e^{-4\Delta^2 / \varepsilon}} \label{eq:sinkSelfLossEE}.
\end{align}

We can now compare the divergences pairwise between the different configurations, starting with the extreme and central ones.
\begin{subequations}
\begin{align}
\SinkDivEps(\alpha_E, \EmpiricalDistDataset) - \SinkDivEps(\alpha_C, \EmpiricalDistDataset) &= \SinkObjEps(\alpha_E, \EmpiricalDistDataset) - \frac{1}{2} \SinkObjEps(\alpha_E, \alpha_E) - W(\alpha_C, \EmpiricalDistDataset) \label{eq:divdiffEC1} \\
&= W(\alpha_E,\EmpiricalDistDataset) + (2p-\frac12) \SinkObjEps(\alpha_E,\alpha_E) - W(\alpha_C, \EmpiricalDistDataset) \label{eq:divdiffEC2} \\
&= \frac{\Delta^2}{2m+1} + (2p-\frac12) \SinkObjEps(\alpha_E,\alpha_E) - \frac{2m\Delta^2}{2m+1} \label{eq:divdiffEC3} \\
&= \frac{2m-1}{2(2m+1)} \left(\SinkObjEps(\alpha_E,\alpha_E) - 2\Delta^2 \right) \label{eq:divdiffEC4} \\
&= \frac{2m-1}{2(2m+1)} \left( \varepsilon \log \frac{2}{1 + e^{-4\Delta^2 / \varepsilon}} - 2\Delta^2 \right) < 0 \label{eq:divdiffEC5}
\end{align}
\end{subequations}
where~\eqref{eq:divdiffEC2} uses \cref{lemma:sinklossexpralphaE},
~\eqref{eq:divdiffEC3} uses the Wasserstein cost expressions from \cref{prop:wasscostunivariate},
~\eqref{eq:divdiffEC4} rearranges the terms and~\eqref{eq:divdiffEC5} uses the expression~\eqref{eq:sinkSelfLossEE} for the symmetric Sinkhorn.
Negativity of the expression comes from $\varepsilon \log \frac{2}{1 + e^{-4\Delta^2 / \varepsilon}} - 2\Delta^2$ being valid for any finite $\varepsilon$.

We now establish the inequality for the extreme and left configurations.
The pair of dual potentials $(f, g)$ with $f= (\Delta^2,0)$,
$g = (-\Delta^2,0,\Delta^2)$ is optimal for the Wasserstein dual problem
$W(\alpha_L, \EmpiricalDistDataset)$, since it results in an objective $\frac{\Delta^2}{2}$
matching the primal and the potentials are feasible. 

Weak duality for the Sinkhorn problem then yields
\begin{align*}
\SinkObjEps(\alpha_L,\EmpiricalDistDataset) \geq & \langle f,\alpha_L\rangle + \langle g, \EmpiricalDist \rangle - \varepsilon\sum_{i,j}(\alpha_L)_i\EmpiricalDistDataset_j\!\left(e^{(f_i + g_j - C_{ij})/\varepsilon}-1\right) \\
   \geq & \frac{\Delta^2}{2} + \varepsilon p (1 - e^{-\frac{2\Delta^2}{\varepsilon}}).
\end{align*}

Using the expression~\eqref{eq:sinkSelfLossLL} for $\SinkObjEps(\alpha_L,\alpha_L)$
and the fact that $\log\left( \frac{2}{1 + e^{- \Delta^2 / \varepsilon}} \right) \leq 1- e^{- \Delta^2 / \varepsilon} $,
we have
\begin{align*}
\frac12 \SinkObjEps(\alpha_L,\alpha_L) \leq \frac{\varepsilon}{2} (1 - e^{-\Delta^2 / \varepsilon}).
\end{align*}
Factoring $1-e^{-2\Delta^2/\varepsilon} = (1-e^{-\Delta^2/\varepsilon})(1+e^{-\Delta^2/\varepsilon})$,
we have
\begin{align*}
   \varepsilon p(1-e^{-2\Delta^2/\varepsilon}) - \tfrac{1}{2}\SinkObjEps(\alpha_L,\alpha_L) \geq \varepsilon(1-e^{-\Delta^2/\varepsilon})\!\left[p(1+e^{-\Delta^2/\varepsilon})-\tfrac{1}{2}\right] \geq \varepsilon(1-e^{-\Delta^2/\varepsilon})\cdot\tfrac{2m-1}{2(2m+1)},
\end{align*}
where the last step uses $p(1+e^{-x}) - \tfrac{1}{2} \geq p - \tfrac{1}{2} = \tfrac{2m-1}{2(2m+1)} > 0$.
Combining the different inequalities, we conclude
\begin{align*}
\SinkDivEps(\alpha_L, \EmpiricalDistDataset) - \SinkDivEps(\alpha_E, \EmpiricalDistDataset) \geq \frac{2m-1}{2(2m+1)} \left(\Delta^2 + \varepsilon (1 - e^{-\Delta^2 / \varepsilon}) - \varepsilon \log \frac{2}{1 + e^{-4\Delta^2 / \varepsilon}} \right).
\end{align*}
Setting $h(t) = t + 1 - e^{-t} - \log\left(\frac{2}{1+e^{-4t}}\right)$, the expression is proportional to $\varepsilon h(\Delta^2 / \varepsilon)$.
We can establish $h(0) = 0$, $h'(0) = 0$, $h'(t) > 0 \, \forall x > 0$, so the expression is positive for $\Delta > 0$. For any value of $\varepsilon$, the divergence of the extreme configuration is lower than that of the left or central configurations, concluding the proof.
\end{proof}

\section{Analysis of the vertex descent algorithms}\label{apx:vertexswap}

In this section, we analyze the performance of \cref{alg:fw-swap},
including the distribution mixture and the analysis of the bias it introduces.
We first derive a directional Lipschitz bound on \SinkDivEps., which we then use to state the main result on the algorithm performance.

\begin{lemma}[Swap-Directional Lipschitz constant]\label{lemma:swapdirlip}
Let $C \in \R^{\npoints \times \npoints}$ be the symmetric cost matrix computed from the pairwise distances of points in the dataset
with $\norm{C}_{\infty} = M$.
Let $\lambda \in (0,1)$ be a mixture parameter and define the mixed distribution as
\begin{align*}
& \alpha^{\lambda}_{\selection} = (1 - \lambda) \EmpiricalDist_{\selection} + \lambda \EmpiricalDistDataset
\quad\text{with weight vector }\quad \va^{\lambda} = \frac{1-\lambda}{\budget} \selection + \frac{\lambda}{\npoints} \1.
\end{align*}

Fix any $\selection \in \HypersimplexVertices$, $\selection' \in \swapset(\selection)$ and set $\vd = \selection' - \selection$.
Define the univariate function
\begin{align*}
\phi_{\vd,\lambda}(t) = \SinkDivEps( \alpha^{\lambda}(\selection + t \vd), \EmpiricalDistDataset ).
\end{align*}
There is a constant $\LipschitzSwap$ depending on $\varepsilon$ and the mixture parameter $\lambda$ such that $\phi_{\vd,\lambda}''(t) \leq \LipschitzSwap$.
\end{lemma}
\begin{proof}
We first show that the second derivative is a Hessian bilinear form and then that
the Hessian is bounded by the Lipschitz constant of the Sinkhorn potentials.

Let $\va^{\lambda}(\selection + t \vd)$ be the vector of weights along the path from $\selection$ to $\selection'$, which is an affine function of $t$ with
$\partial \va^{\lambda} / \partial t = (1-\lambda) \vd / \budget$,
and define $h: \Simplex{\npoints} \to \R_+$, $h(\va) = \SinkDivEps(\sum_{i\in[\npoints]} \va_i \delta_{\vy_i}, \EmpiricalDistDataset)$.

By the chain rule,
\begin{align*}
\phi_{\vd,\lambda}''(t) = \frac{(1 - \lambda)^2}{\budget^2} \vd^\top \nabla^2 h(\va^{\lambda}(\selection + t \vd)) \vd.
\end{align*}

We recall from Equation~\eqref{eq:sinkdivgradient} that the gradient of $h$ w.r.t.~the weights is given by $\nabla h(\va) = \PotentialSinkf^*_{\va} - \PotentialSinkp^*_{\va}$.

Define the Sinkhorn marginal Lipschitz constant
\begin{align*}
L_\varepsilon^{\mathrm{f,p}}(\lambda_0) := \sup_{\substack{\va, \va' \in \Simplex{\npoints}\\ \min_k \va_k, \va'_k \geq \lambda_0}}\,\frac{\|\PotentialSinkf^\star_{\va} - \PotentialSinkf^\star_{\va'} \|_\infty + \|\PotentialSinkp^\star_{\va} - \PotentialSinkp^\star_{\va'} \|_\infty}{\|\va - \va'\|_1},
\end{align*}
which is well-defined and finite, from the implicit function theorem on the system of equations defining the fixed point of the Sinkhorn algorithm see, e.g., in \citet{nutz2023stability,deligiannidis2024quantitative,carlier2024displacement} for results on the stability of Sinkhorn potentials w.r.t.~their marginals up to normalization.
This quantity depends on $\exp(M/\varepsilon)$ which is not problematic in our setup considering large regularization and a small diameter.

The swap direction has only two nonzero entries, thus
\begin{align*}
\vd^\top \nabla^2 h(\va) d & = H_{ii} + H_{jj} - 2H_{ij} \\
|\vd^\top \nabla^2 h(\va) \vd| & \leq 4 \max_{(i,j) \in [\npoints] \times [\npoints]} |H_{ij}|
\end{align*}
with $H$ the Hessian of $h$ at $\va$.
Each entry is bounded in absolute value by the Lipschitz constant of $\nabla h$,
and using $\lambda_0 = \lambda / \npoints$, the minimum weight from mixture, we have
\begin{align*}
\vd^\top \nabla^2 h(\va) d &\leq 4 L_\varepsilon^{\mathrm{f,p}}(\lambda / \npoints) \\
\phi_{\vd,\lambda}''(t) & \leq \LipschitzSwap = \frac{4 (1-\lambda)^2 L_\varepsilon^{\mathrm{f,p}}(\lambda / \npoints)}{\budget^2}.
\end{align*}
\end{proof}

\begin{lemma}[Potential bound on vertices]\label{lemma:potentialboundatvtx}
Let $C \in \R^{\npoints\times \npoints}$ be the symmetric cost matrix with $\norm{C}_{\infty} \leq M$,
then the following bound on the optimal potentials holds at any $\selection$ in the relative interior of $\Hypersimplex$, under the normalization $\sum_{i} \PotentialSinkf_i = 0$:
\begin{align}\label{eq:potentialinfbound}
\norm{\PotentialSinkf - \PotentialSinkp}_{\infty} \leq 2M.
\end{align}
\end{lemma}
\begin{proof}
Using the expression of the optimal plan $\TransportPlan$ of the EOT problem from the potentials:
\begin{align*}
\TransportPlan_{ij} = e^{(\PotentialSinkf_i + \PotentialSinkg_j - C_{ij})/\varepsilon} \frac{\selection_i}{\npoints \budget} \; \forall i \in [\npoints], j \in [\npoints],
\end{align*}
the fact that $\selection_i > 0 \, \forall i \in [\npoints]$,
and the marginal constraint $\sum_{j\in[\npoints]} \TransportPlan_{ij} = \frac{\selection_i}{\budget}$
results in
\begin{align*}
& \sum_{j\in[\npoints]} e^{(\PotentialSinkf_i + \PotentialSinkg_j - C_{ij})/\varepsilon} \frac{\selection_i}{\npoints \budget} = \frac{\selection_i}{\budget} \quad \forall i \in [\npoints] \\
& e^{\PotentialSinkf_i / \varepsilon} \sum_{j\in[\npoints]} e^{(\PotentialSinkg_j - C_{ij})/\varepsilon} = \npoints \quad \forall i \in [\npoints] \\
& \PotentialSinkf_i - \PotentialSinkf_{i'} = -\varepsilon \log \frac{\sum_j e^{(\PotentialSinkg_j - C_{ij})/\varepsilon}}{\sum_j e^{(\PotentialSinkg_j - C_{i'j})/\varepsilon}} \\
& \max_{i,i'} |\PotentialSinkf_{i'} - \PotentialSinkf_{i} | \leq M
\end{align*}
where the last inequality uses cancellation of the $e^{\PotentialSinkg_j / \varepsilon}$ terms and $C_{ij} \leq M$.
Under the normalization $\sum_i \PotentialSinkf_i = 0$ and
applying the same derivation to the self-transport problem for $\PotentialSinkp$, we obtain the desired inequality~\eqref{eq:potentialinfbound}
\end{proof}

\swapdescent*
\begin{proof}
The swap descent algorithm terminates when the new iterate was already visited in a previous iteration, creating a cycle in the vertex graph of the hypersimplex.
On this cycle, let $\selection^{(c)}$ be the minimizer of the function and $\selection^{(c+1)}$ be its immediate successor in the cycle,
and set $\vd_c = \selection^{(c+1)} - \selection^{(c)}$.
We will denote $M = \norm{C}_{\infty}$ throughout the proof.

Define $f_{\lambda}: \Hypersimplex \to \R_+$, $f_{\lambda}(\selection) = \SinkDivEps( \alpha_{\lambda}(\selection), \EmpiricalDistDataset )$,
which is convex, and note that $f_{\lambda}(\selection + t \vd) = \phi_{\vd,\lambda}(t)$ if $\vd$ is a swap direction, with $\phi_{\vd,\lambda}$ defined as in \cref{lemma:swapdirlip}.
The Lipschitz smoothness along $\vd_c$ yields
\begin{align*}
f_{\lambda}(\selection^{(c+1)}) & \leq f_{\lambda}(\selection^{(c)}) + \innp{\nabla f_{\lambda}(\selection^{(c)})}{\vd_c} + \frac{\LipschitzSwap}{2}.
\end{align*}
Using the cycle-optimality of $\selection^{(c)}$, and re-arranging,
\begin{align*}
& 0 \leq f_{\lambda}(\selection^{(c+1)}) - f_{\lambda}(\selection^{(c)}) \leq \innp{\nabla f_{\lambda}(\selection^{(c)})}{\vd_c} + \frac{\LipschitzSwap}{2} \\
& \innp{\nabla f_{\lambda}(\selection^{(c)})}{\vd_c} \geq -\frac{\LipschitzSwap}{2}.
\end{align*}
The last inequality also holds for any swap direction by optimality of $\vd_c$.
By convexity, we have
\begin{align*}
   f_{\lambda}(\selection^{(c)}) - f_{\lambda}(\selection^*) \leq -\innp{\nabla f_\lambda(\selection^{(c)})}{\selection^* - \selection^{(c)}}
\end{align*}
Decompose $\selection^*$ into $k^* \leq \budget$ swaps from $\selection^{(c)}$:
\begin{align*}
\selection^* - \selection^{(c)} = \sum_{l=1}^{k^*} \vd_l
\end{align*}
with each $\vd_l$ corresponding to a single swap from the support of $\selection^{(c)}$ to the support of $\selection^*$.
We have
\begin{align}\label{eq:suboptimalityswap}
f_{\lambda}(\selection^{(c)}) - f_{\lambda}(\selection^*) \leq \sum_{l=1}^{k^*} -\innp{\nabla f_{\lambda}(\selection^{(c)})}{\vd_l} \leq k^* \frac{\LipschitzSwap}{2}.
\end{align}
Conservatively, we can bound $k^* \leq \budget$ for an instance-independent guarantee.

Finally, we bound the bias introduced by the mixture parameter $\lambda$.
We start by noting that $\SinkDivEps(\alpha_{\selection}^\lambda, \EmpiricalDistDataset)$ is a convex decreasing function of lambda on $[0,1]$.
Using this fact with the expression of the gradient of the Sinkhorn divergence
with respect to the weights, and the continuity of
$\SinkDivEps(\cdot,\EmpiricalDistDataset)$,
\begin{align*}
\SinkDivEps(\alpha^{\lambda}_{\selection}, \EmpiricalDistDataset) - \SinkDivEps(\alpha^{0}_{\selection}, \EmpiricalDistDataset) = & \lim_{\tau \to 0} \lambda \int_{\tau}^1 \innp{\PotentialSinkf_t - \PotentialSinkp_t }{\EmpiricalDistDataset - \alpha^{0}_{\selection}} dt \\
|\SinkDivEps(\alpha^{\lambda}_{\selection}, \EmpiricalDistDataset) - \SinkDivEps(\alpha^{0}_{\selection}, \EmpiricalDistDataset) | \leq & \lim_{\tau \to 0} \lambda \norm{\EmpiricalDistDataset - \alpha^{\tau}_{\selection}}_{\mathrm{TV}} \norm{\PotentialSinkf_\tau - \PotentialSinkp_\tau }_{\infty} \\
|\SinkDivEps(\alpha^{\lambda}_{\selection}, \EmpiricalDistDataset) - \SinkDivEps(\alpha^{0}_{\selection}, \EmpiricalDistDataset) | \leq & 2\lambda M \norm{\EmpiricalDistDataset - \alpha^{0}_{\selection}}_{\mathrm{TV}}
\end{align*}
with $\PotentialSinkf_t$, $\PotentialSinkp_t$, $\PotentialSinkf$, $\PotentialSinkp$ the potentials of the Sinkhorn divergence evaluated at $\alpha^{t}_{\selection}$,
$\alpha^{0}_{\selection}$ respectively.
The first step uses the fact that the divergence is convex and monotonically decreasing in $t$, such that its derivative w.r.t.~and attains its greatest absolute value at $t=\tau$, together with Hölder's inequality.
Since $\alpha^{0}_{\selection}$ is a measure supported on $\budget$ points,
$\norm{\EmpiricalDistDataset - \alpha^0_{\selection}}_{\mathrm{TV}} = 2 \left(1 - \frac{\budget}{\npoints} \right)$.
We can therefore bound the mixing bias by
\begin{align}\label{eq:biascontrollambda}
|\SinkDivEps(\alpha^{\lambda}_{\selection}, \EmpiricalDistDataset) - \SinkDivEps(\alpha^{0}_{\selection}, \EmpiricalDistDataset) | \leq & 4\lambda M \left(1 - \frac{\budget}{\npoints} \right).
\end{align}
We can combine inequality~\eqref{eq:suboptimalityswap} with the control of the mixture bias
~\eqref{eq:biascontrollambda} applied at $\bar\selection$, together with the fact
that the bias is non-positive at $\selection^*$ and can be dropped from the inequality,
providing the desired result~\eqref{eq:swapsuboptimality}.
\end{proof}

Empirically, we perform multiple runs on a synthetic dataset of 500 points and 2 classes (in order to be able to compute the true minimizer by complete enumeration of the hypersimplex). We present the results in \cref{tab:gap-true}, where the gap is that of \FWSwap, Rand-Swap is the swap algorithm initialized with a random vertex, and the random subset is the average divergence of vertices.

\begin{table}[ht!]
\tiny
\caption{Comparison of the objective values obtained by different subset selections over three runs.}\label{tab:gap-true}
\centering
\begin{tabular}{lccccc}
\toprule
\textbf{Run}
& \textbf{True Minimum}
& \textbf{FW-Swap}
& \textbf{Gap}
& \textbf{Rand-Swap}
& \textbf{Random Subset} \\
\midrule
1 & $0.02316$ & $0.02549$ & $0.00233$ & $0.04523$ & $0.06378$ \\
2 & $0.02522$ & $0.02794$ & $0.00272$ & $0.04341$ & $0.14044$ \\
3 & $0.02436$ & $0.02676$ & $0.00239$ & $0.03632$ & $0.12600$ \\
\bottomrule
\end{tabular}
\end{table}

We see a small gap between the true minimum and the one obtained from \FWSwap, proving the effectiveness of the algorithm. Furthermore, \cref{tab:gap-true} shows the effect of the FW initialization compared to the random one as it outputs closer values to the true minimum.

\section{Comparison of the self-supervised learning representation}\label{apx:simclr-vs-dino}
As mentioned, the choice of the feature space heavily impacts the results. 
In fact, when comparing a traditional SSL model, SimCLR, to DINOv3, we notice a difference across all methods in performance. 
In particular, for \FWSwap, who relies on the geometry of this feature space.
We provide below in \cref{tab:stl-simclr-results,tab:cifar-simclr-results} the results obtained for STL-10 and CIFAR-10 when training the active learning model on features extracted from SimCLR.

\begin{table}[hbt!]
\tiny
  \caption{Test accuracy (mean $\pm$ std) across active learning methods and budgets on STL-10 with feature space defined by SimCLR.}\label{tab:stl-simclr-results}
  \centering
    \begin{tabular}{lcccccc}
    \toprule
    \textbf{Method} & \textbf{10} & \textbf{20} & \textbf{30} & \textbf{40} & \textbf{50} & \textbf{60} \\
    \midrule
    FW-Swap  & $\mathbf{0.521 \pm 0.018}$ & $\mathbf{0.569 \pm 0.008}$ & $0.611 \pm 0.008$ & $0.613 \pm 0.020$ & $0.652 \pm 0.012$ & $0.625 \pm 0.006$ \\
    \midrule
    MaxHerding & $0.462 \pm 0.042$ & $0.547 \pm 0.027$ & $0.605 \pm 0.032$ & $\mathbf{0.642 \pm 0.032}$ & $\mathbf{0.674 \pm 0.021}$ & $0.679 \pm 0.021$ \\
    TypiClust  & $0.483 \pm 0.035$ & $0.553 \pm 0.047$ & $0.614 \pm 0.020$ & $0.637 \pm 0.029$ & $0.664 \pm 0.025$ & $\mathbf{0.691 \pm 0.016}$ \\
    Medoid      & $0.440 \pm 0.062$ & $0.566 \pm 0.057$ & $\mathbf{0.635 \pm 0.021}$ & $0.638 \pm 0.031$ & $0.669 \pm 0.026$ & $0.683 \pm 0.013$ \\
    k-Center    & $0.444 \pm 0.037$ & $0.568 \pm 0.028$ & $0.612 \pm 0.037$ & $0.631 \pm 0.027$ & $0.673 \pm 0.020$ & $0.688 \pm 0.015$ \\
    ProbCover   & $0.432 \pm 0.010$ & $0.483 \pm 0.006$ & $0.574 \pm 0.003$ & $0.598 \pm 0.012$ & $0.653 \pm 0.008$ & $0.679 \pm 0.005$ \\
    Wasserstein & $0.402 \pm 0.047$ & $-$               & $-$               & $-$               & $-$               & $-$               \\
    Random      & $0.383 \pm 0.051$ & $0.514 \pm 0.045$ & $0.579 \pm 0.025$ & $0.618 \pm 0.040$ & $0.659 \pm 0.024$ & $0.670 \pm 0.028$ \\
    \bottomrule
    \end{tabular}
\end{table}

The Wasserstein method is dashed after $\budget=20$ in \cref{tab:stl-simclr-results} because each run had to be terminated after 3 hours, as reported in the original paper.

\begin{table}[hbt!]
\tiny
  \caption{Test accuracy (mean $\pm$ std) across active learning methods and budgets on CIFAR-10 with feature space defined by SimCLR.}\label{tab:cifar-simclr-results}
  \centering
    \begin{tabular}{lcccccc}
    \toprule
    \textbf{Method} & \textbf{10} & \textbf{20} & \textbf{30} & \textbf{40} & \textbf{50} & \textbf{60} \\
    \midrule
    FW-Swap   & $\mathbf{0.537 \pm 0.016}$ & $0.606 \pm 0.018$ & $0.639 \pm 0.010$ & $0.683 \pm 0.020$ & $\mathbf{0.744 \pm 0.009}$ & $0.749 \pm 0.008$ \\
    \midrule
    MaxHerding  & $0.468 \pm 0.027$ & $0.601 \pm 0.023$ & $0.689 \pm 0.008$ & $0.704 \pm 0.014$ & $0.734 \pm 0.011$ & $\mathbf{0.763 \pm 0.007}$ \\
    TypiClust  & $0.477 \pm 0.040$ & $0.553 \pm 0.035$ & $0.639 \pm 0.034$ & $0.668 \pm 0.028$ & $0.693 \pm 0.041$ & $0.732 \pm 0.012$ \\
    Medoid      & $0.518 \pm 0.044$ & $\mathbf{0.630 \pm 0.043}$ & $\mathbf{0.714 \pm 0.033}$ & $\mathbf{0.723 \pm 0.020}$ & $0.740 \pm 0.025$ & $0.757 \pm 0.022$ \\
    k-Center    & $0.444 \pm 0.059$ & $0.610 \pm 0.062$ & $0.670 \pm 0.037$ & $0.715 \pm 0.037$ & $0.723 \pm 0.025$ & $0.741 \pm 0.017$ \\
    ProbCover   & $0.454 \pm 0.014$ & $0.561 \pm 0.012$ & $0.633 \pm 0.011$ & $0.670 \pm 0.026$ & $0.721 \pm 0.010$ & $0.746 \pm 0.005$ \\
    Random      & $0.384 \pm 0.073$ & $0.531 \pm 0.044$ & $0.604 \pm 0.013$ & $0.631 \pm 0.042$ & $0.684 \pm 0.033$ & $0.732 \pm 0.018$ \\
    \bottomrule
    \end{tabular}
\end{table}

\section{Detailed experimental results}\label{apx:additionalexperiments}

This section presents the tables corresponding to the experiments of \cref{sec:emp-results}. \cref{tab:stl-10-full,tab:svhn-full,tab:gas-full} show the detailed results of \cref{fig:stl-results,fig:svhn-results,fig:gashis-results} respectively.

\begin{table}[ht!]
\tiny
  \caption{Test accuracy (mean $\pm$ std) across active learning methods and budgets on STL-10.}\label{tab:stl-10-full}
  \centering
    \begin{tabular}{lcccccc}
    \toprule
    \textbf{Method} & \textbf{10} & \textbf{20} & \textbf{30} & \textbf{40} & \textbf{50} & \textbf{60} \\
    \midrule
    FW-Swap    & $\mathbf{0.699 \pm 0.008}$ & $\mathbf{0.768 \pm 0.011}$ & $\mathbf{0.829 \pm 0.012}$ & $0.852 \pm 0.016$ & $0.879 \pm 0.008$ & $0.906 \pm 0.004$ \\
    \midrule
    MaxHerding & $0.612 \pm 0.033$ & $0.719 \pm 0.037$ & $0.813 \pm 0.023$ & $0.846 \pm 0.029$ & $0.890 \pm 0.018$ & $0.899 \pm 0.018$ \\
    TypiClust  & $0.617 \pm 0.030$ & $0.735 \pm 0.031$ & $0.805 \pm 0.038$ & $0.851 \pm 0.023$ & $0.890 \pm 0.014$ & $\mathbf{0.911 \pm 0.014}$ \\
    Medoid     & $0.606 \pm 0.026$ & $0.714 \pm 0.058$ & $0.797 \pm 0.035$ & $0.841 \pm 0.025$ & $0.869 \pm 0.042$ & $0.902 \pm 0.022$ \\
    k-Center   & $0.648 \pm 0.019$ & $0.749 \pm 0.029$ & $0.809 \pm 0.030$ & $\mathbf{0.861 \pm 0.018}$ & $\mathbf{0.899 \pm 0.011}$ & $0.910 \pm 0.006$ \\
    ProbCover  & $0.668 \pm 0.015$ & $0.755 \pm 0.018$ & $0.750 \pm 0.015$ & $0.739 \pm 0.022$ & $0.743 \pm 0.023$ & $0.724 \pm 0.013$ \\
    Random     & $0.471 \pm 0.060$ & $0.603 \pm 0.033$ & $0.729 \pm 0.047$ & $0.750 \pm 0.042$ & $0.797 \pm 0.043$ & $0.854 \pm 0.026$ \\
    \bottomrule
    \end{tabular}
\end{table}

\begin{table}[ht!]
\tiny
  \caption{Test accuracy (mean $\pm$ std) across active learning methods and budgets on SVHN.}\label{tab:svhn-full}
  \centering
    \begin{tabular}{lcccccc}
    \toprule
    \textbf{Method} & \textbf{10} & \textbf{20} & \textbf{30} & \textbf{40} & \textbf{50} & \textbf{60} \\
    \midrule
    FW-Swap    & $\mathbf{0.307 \pm 0.002}$ & $\mathbf{0.496 \pm 0.004}$ & $\mathbf{0.497 \pm 0.003}$ & $\mathbf{0.576 \pm 0.004}$ & $0.577 \pm 0.002$ & $\mathbf{0.648 \pm 0.003}$ \\
    \midrule
    MaxHerding & $0.110 \pm 0.002$ & $0.307 \pm 0.003$ & $0.468 \pm 0.003$ & $0.516 \pm 0.002$ & $\mathbf{0.598 \pm 0.002}$ & $0.606 \pm 0.004$ \\
    TypiClust  & $0.283 \pm 0.010$ & $0.402 \pm 0.026$ & $0.447 \pm 0.048$ & $0.509 \pm 0.042$ & $0.563 \pm 0.029$ & $0.599 \pm 0.053$ \\
    Medoid     & $0.245 \pm 0.038$ & $0.353 \pm 0.076$ & $0.437 \pm 0.055$ & $0.563 \pm 0.057$ & $0.575 \pm 0.037$ & $0.635 \pm 0.047$ \\
    k-Center   & $0.156 \pm 0.043$ & $0.226 \pm 0.034$ & $0.278 \pm 0.038$ & $0.326 \pm 0.034$ & $0.376 \pm 0.036$ & $0.416 \pm 0.038$ \\
    ProbCover  & $0.256 \pm 0.004$ & $0.309 \pm 0.002$ & $0.369 \pm 0.002$ & $0.430 \pm 0.005$ & $0.477 \pm 0.003$ & $0.530 \pm 0.003$ \\
    Random     & $0.261 \pm 0.065$ & $0.352 \pm 0.056$ & $0.428 \pm 0.050$ & $0.451 \pm 0.039$ & $0.539 \pm 0.064$ & $0.600 \pm 0.051$ \\
    \bottomrule
    \end{tabular}
\end{table}

\begin{table}[!t]
\tiny
\centering
  \caption{Test accuracy (mean $\pm$ std) across active learning methods and budgets on STL-10.}
    \begin{tabular}{lcc}
    \toprule
    \textbf{Method} & \textbf{80} & \textbf{100} \\
    \midrule
    FW-Swap    & $\mathbf{0.721 \pm 0.002}$ & $\mathbf{0.752 \pm 0.002}$ \\
    \midrule
    MaxHerding & $0.667 \pm 0.002$ & $0.723 \pm 0.003$ \\
    TypiClust & $0.666 \pm 0.025$ & $0.708 \pm 0.015$ \\
    Medoid     & $0.693 \pm 0.031$ & $0.715 \pm 0.028$ \\
    k-Center   & $0.502 \pm 0.037$ & $0.581 \pm 0.025$ \\
    ProbCover  & $0.622 \pm 0.002$ & $0.691 \pm 0.003$ \\
    Random     & $0.641 \pm 0.027$ & $0.678 \pm 0.027$ \\
    \bottomrule
    \end{tabular}
\end{table}

\begin{table}[!t]
\tiny
  \caption{Test accuracy (mean $\pm$ std) across active learning methods and budgets on GasHisSDB.}\label{tab:gas-full}
  \centering
    \begin{tabular}{lccccc}
    \toprule
    \textbf{Method} & \textbf{2} & \textbf{4} & \textbf{6} & \textbf{8} & \textbf{10} \\
    \midrule
    FW-Swap    & $\mathbf{0.665 \pm 0.003}$ & $0.698 \pm 0.004$ & $\mathbf{0.820 \pm 0.002}$ & $\mathbf{0.809 \pm 0.002}$ & $\mathbf{0.838 \pm 0.002}$ \\
    \midrule
    MaxHerding & $0.606 \pm 0.001$ & $0.605 \pm 0.002$ & $0.609 \pm 0.003$ & $0.723 \pm 0.004$ & $0.769 \pm 0.002$ \\
    TypiClust & $0.632 \pm 0.008$ & $0.691 \pm 0.058$ & $0.734 \pm 0.032$ & $0.730 \pm 0.061$ & $0.711 \pm 0.077$ \\
    Medoid     & $0.598 \pm 0.116$ & $0.693 \pm 0.067$ & $0.754 \pm 0.074$ & $0.806 \pm 0.061$ & $0.803 \pm 0.047$ \\
    k-Center   & $0.601 \pm 0.070$ & $0.668 \pm 0.061$ & $0.715 \pm 0.075$ & $0.754 \pm 0.044$ & $0.770 \pm 0.040$ \\
    ProbCover  & $0.505 \pm 0.008$ & $\mathbf{0.718 \pm 0.005}$ & $0.695 \pm 0.005$ & $0.704 \pm 0.007$ & $0.688 \pm 0.005$ \\
    Random     & $0.570 \pm 0.121$ & $0.622 \pm 0.079$ & $0.702 \pm 0.105$ & $0.685 \pm 0.107$ & $0.742 \pm 0.076$ \\
    \bottomrule
    \end{tabular}
\end{table}

The detailed results reported in the tables above further support our observations that $\FWSwap$ consistently performs strongly at low budgets. It achieves the highest test accuracy in most cases and remains competitive with the strongest-performing methods when it does not.


\end{document}